%% file: mairal13a.tex
\documentclass[letter,twoside,11pt]{article}
\usepackage{pslatex}
\usepackage{amsmath}
\usepackage{amsfonts}
\usepackage{paralist}
\usepackage{amssymb}
\usepackage{algorithm,algorithmic}
\usepackage{multirow}
\usepackage{color}
\usepackage{pgf}
\usepackage{url}
\usepackage{tikz}
\usepackage{textcomp}
\usetikzlibrary{arrows,shapes,snakes,automata,backgrounds,petri}
\graphicspath{{./images/}}
\usepackage{subfig}
\usepackage[hyperindex,breaklinks]{hyperref}
\usepackage{jmlr2e}

\input{macros.tex}

\jmlrheading{14}{2013}{-}{04/12; Revised 03/13}{-}{Julien Mairal and Bin Yu}

\ShortHeadings{Feature Selection in Graphs with Path Coding Penalties}{Mairal and Yu}
\firstpageno{1}

\begin{document}

\title{Supervised Feature Selection in Graphs with Path Coding Penalties \\ and Network Flows}

\author{\name Julien Mairal\thanks{Present address: LEAR Project-Team, INRIA Grenoble Rh\^one-Alpes, France.} \email julien.mairal@inria.fr \\
        \name Bin Yu\thanks{Also in the department of Electrical Engineering \& Computer Science.}  \email binyu@stat.berkeley.edu \\
        \addr Department of Statistics \\
University of California\\
Berkeley, CA 94720-1776, USA.}

\editor{Ben Taskar}

\maketitle

\begin{abstract}
   \input{abstract.tex}

\end{abstract}

\begin{keywords}
convex and non-convex optimization, network flow optimization,  graph sparsity
\end{keywords}

\section{Introduction}\label{sec:intro}
\input{intro.tex}
\section{Preliminaries} \label{sec:related}
\input{related.tex}

\section{Sparse Estimation in Graphs with Path Coding Penalties}\label{sec:approach}
\input{approach.tex}

\input{optim.tex}
\section{Experiments and Applications} \label{sec:exp}
\input{exp.tex}

\section{Conclusion}\label{sec:ccl}
\input{ccl.tex}

\acks{This paper was supported in part by NSF grants SES-0835531, CCF-0939370,
DMS-1107000, DMS-0907632, and by ARO-W911NF-11-1-0114.  Julien Mairal would
like to thank Laurent Jacob, Rodolphe Jenatton, Francis Bach, Guillaume
Obozinski and Guillermo Sapiro for interesting discussions and suggestions
leading to improvements of this manuscript, and his former research lab, the
INRIA WILLOW and SIERRA project-teams, for letting him use computational
resources funded by the European Research Council (VideoWorld and
Sierra projects). He would also like to thank Junzhou Huang for providing the
source code of his StructOMP software.
}

\appendix
  
\input{appendix.tex}

\bibliography{mairal13a}

\end{document}

%% file: macros.tex
\def\x{{\mathbf x}}

\def\1{{\mathbf 1}}

\def\X{{\mathbf X}}

\def\hatw{{\bf \hat w}}

\def\varepsilonb{{\boldsymbol\varepsilon}}

\def\y{{\mathbf y}}
\def\w{{\mathbf w}}
\def\tildew{{\mathbf {\tilde w}}}

\def\NN{{\mathbf N}}

\def\GG{{\mathcal G}}
\def\JJ{{\mathcal J}}

\def\u{{\mathbf u}}

\def\Real{{\mathbb R}}

\def\u{{\mathbf u}}

\def\FF{{\mathcal F}}

\def\argmin{\operatornamewithlimits{arg\,min}}

\def\sign{\operatorname{sign}}
\def\st{~~\text{s.t.}~~}
\def\defin{\triangleq}

\newcommand{\R}[1]{\mathbb{R}^{#1}}

\newcommand{\G}{\mathcal{G}}

\def \xib{{\boldsymbol\xi}}

\def \kappab {{\boldsymbol\kappa}}

\long\def\symbolfootnote[#1]#2{\begingroup\def\thefootnote{\fnsymbol{footnote}}\footnote[#1]{#2}\endgroup} 

\newenvironment{myitemize}
{
   \begin{list}{\labelitemi}{}
   \setlength{\parskip}{1pt}
   \setlength{\parsep}{1pt}         
   \setlength{\itemsep}{3pt} 
   \setlength{\topsep}{1pt}
   \setlength{\partopsep}{1pt}
}
{
   \end{list} 
}

\let\oldvarphi\varphi
\let\oldpsi\psi

\renewcommand{\varphi}{\oldvarphi_{\GG}}
\renewcommand{\psi}{\oldpsi_{\GG}}
\newcommand{\varphip}{\oldvarphi_{\GG_p}}
\newcommand{\psip}{{\oldpsi}_{{\GG}_p}}

\renewcommand{\cite}{\citep}

%% file: abstract.tex
We consider supervised learning problems where the features are embedded in a
graph, such as gene expressions in a gene network. In this context, it is of
much interest to automatically select a subgraph with few connected components;
by exploiting prior knowledge, one can indeed improve the prediction
performance or obtain results that are easier to interpret. Regularization or penalty
functions for selecting features in graphs have recently been proposed, but
they raise new algorithmic challenges. For example, they typically require
solving a combinatorially hard selection problem among all connected subgraphs.
In this paper, we propose computationally feasible strategies to select a
sparse and well-connected subset of features sitting on a directed acyclic
graph (DAG). We introduce structured sparsity penalties over paths on a DAG
called ``path coding'' penalties. Unlike existing regularization functions
that model long-range interactions between features in a graph, path coding
penalties are tractable. The penalties and their proximal operators involve
path selection problems, which we efficiently solve by leveraging
network flow optimization. We experimentally show on synthetic, image, and
genomic data that our approach is scalable and leads to more connected
subgraphs than other regularization functions for graphs.

%% file: intro.tex
Supervised sparse estimation problems have been the topic of much research in
statistical machine learning and signal processing. In high dimensional
settings, restoring a signal or learning a model is
often difficult without a priori knowledge. When the solution is
known beforehand to be sparse---that is, has only a few non-zero
coefficients, regularizing with sparsity-inducing penalties has been shown 
to provide better prediction and solutions that are easier to interpret. For that purpose,
non-convex penalties and greedy algorithms have been
proposed~\citep[][]{akaike,schwarz,rissanen,mallat4,fan}. More recently,
convex relaxations such as the $\ell_1$-norm~\citep{tibshirani,chen} and
efficient algorithms have been
developed~\citep[][]{osborne,nesterov,beck,wright}.

In this paper, we consider supervised learning problems where more information
is available than just sparsity of the solution. More precisely, we assume that
the features (or predictors) can be identified to the vertices of a graph, such
as gene expressions in a gene network. In this context, it can be desirable to
take into account the graph structure in the regularization~\citep{rapaport}.
In particular, we are interested in automatically identifying a subgraph with few connected
components~\citep{jacob,huang}, groups of genes involved in a disease for
example. There are two equally important reasons for promoting the connectivity
of the problem solution: either connectivity is a prior information, which might
improve the prediction performance, or connected components may be easier to
interpret than isolated variables.

Formally, let us consider a supervised sparse estimation problem involving~$p$
features, and let us assume that we are given an undirected or directed
graph~$G=(V,E)$, where $V$ is a vertex set identified
to~$\{1,\ldots,p\}$, and~$E \subseteq V \times V$ is an arc (edge) set.
Classical empirical risk minimization problems can be formulated as
\begin{equation} 
\min_{\w \in \Real^p} [ L(\w) + \lambda \Omega(\w)],\label{eq:prob} 
\end{equation} 
where $\w$ is a weight vector in~$\Real^p$, which we wish to estimate; $L: \Real^p \to
\Real$ is a convex loss function, and~$\Omega:\Real^p \to \Real$ is a
regularization function. In order to obtain a sparse
solution, $\Omega$ is often chosen to be the $\ell_0$- (cardinality of the support) or~$\ell_1$-penalty.
In this paper, we are also interested in encouraging the
sparsity pattern of~$\w$ (the set of non-zero coefficients) to form a
subgraph of~$G$ with few connected components. 

To the best of our knowledge, penalties promoting the connectivity of sparsity
patterns in a graph can be classified into two categories. The ones of the first
category involve pairwise interactions terms between vertices
linked by an arc~\citep{cehver,jacob,chen3}; each term encourages two
neighbors in the graph to be simultaneously selected. Such
regularization functions usually lead to tractable optimization problems, but they do not model
long-range interactions between variables in the graph, and they do not promote
large connected components. Penalties from the second category are more
complex, and directly involve hard combinatorial problems~\citep{huang}. As such, they cannot be used without approximations. The problem of finding tractable penalties that model long-range
interactions is therefore acute. The main
contribution of our paper is a solution to this problem when the graph is \emph{directed and
acyclic}.

Of much interest to us are the non-convex penalty of~\citet{huang} and the convex penalty of~\citet{jacob}.  Given a pre-defined set of
possibly overlapping groups of
variables~$\G$, these two structured sparsity-inducing regularization functions
encourage a sparsity pattern to be \emph{in the union of a small number of
groups from~$\GG$}.  Both penalties induce a similar regularization effect and are
strongly related to each other.  In fact, we show in Section~\ref{sec:approach} that the
penalty of~\citet{jacob} can be interpreted as a convex relaxation of the
non-convex penalty of~\citet{huang}.
These two penalties go beyond classical unstructured sparsity,
but they are also complex and 
raise new challenging combinatorial problems. 
For example, \citet{huang} define~$\G$ as the set of all connected
subgraphs of~$G$, which leads to well-connected solutions but also leads to
intractable optimization problems; the latter are approximately addressed
by~\citet{huang} with greedy algorithms.
\citet{jacob} choose a different strategy and define~$\G$ as the pairs of vertices linked by an arc,
which, as a result, encourages neighbors in the graph to be simultaneously
selected. This last formulation is computationally tractable, but does not
model long-range interactions between features. Another
suggestion from~\citet{jacob} and \citet{huang} consists of
defining~$\G$ as the set of connected subgraphs up to a size~$k$. The
number of such subgraphs is however exponential in~$k$, making this approach difficult
to use even for small subgraph sizes ($k\!=\!3,4$) as soon as the graph is
large~($p\! \approx \!  10\,000$) and connected enough.\footnote{This issue was
confirmed to us in a private communication with Laurent Jacob, and this was one
of our main motivation for developing new algorithmic tools overcoming this
problem.} These observations naturally raise the question: \emph{can we replace
connected subgraphs by another structure that is rich enough to model
long-range interactions in the graph and leads to computationally
feasible penalties?}

When the graph~$G$ is directed and acyclic, we propose a solution built upon two ideas. First, we use in the
penalty framework of~\citet{jacob} and~\citet{huang} a novel group
structure~$\G_p$ that contains \emph{all the paths} in~$G$;
a path is defined as a sequence of vertices $(v_1,\ldots,v_k)$ such that
for all $1\!\leq\! i < \!k$, we have $(v_i,v_{i+1})\! \in\! E$.
The second idea is to use appropriate costs for
each path (the ``price'' one has to pay to select a path), which, as we show in
the sequel, allows us to leverage network flow optimization.
We call the resulting regularization functions ``path coding'' penalties. They
go beyond pairwise interactions between vertices and model long-range
interactions between the variables in the graph. They encourage sparsity
patterns forming subgraphs that can be covered by a small number of paths,
therefore promoting connectivity of the solution.  We illustrate the ``path
coding'' concept for DAGs in Figure~\ref{fig:graphsparsity}.
Even though the number of paths in a DAG is exponential in the graph size, we
map the \emph{path selection} problems our penalties involve to network flow
formulations~\citep[see][]{ahuja,bertsekas2}, which can be solved in polynomial
time.  As shown in Section~\ref{sec:approach}, we build
minimum cost flow formulations such that sending a positive amount of flow
along a path for minimizing a cost is equivalent to selecting the path. This
allows us to efficiently compute the penalties and their proximal operators, a
key tool to address regularized problems~\citep[see][for a review]{bach8}.

\tikzstyle{source}=[circle,thick,draw=blue!75,fill=blue!20,minimum size=8mm]
\tikzstyle{sink}=[circle,thick,draw=blue!75,fill=blue!20,minimum size=8mm]
\tikzstyle{group}=[place,thick,draw=red!75,fill=red!20, minimum size=8mm]
\tikzstyle{groupwhite}=[place,thick,draw=black!100,fill=red!0, minimum size=8mm]
\tikzstyle{groupgray}=[place,thick,draw=black!100,fill=gray!60, minimum size=8mm]
\tikzstyle{empty}=[place,draw=red!0, minimum size=8mm]
\tikzstyle{var}=[rectangle,thick,draw=black!75,fill=black!20,minimum size=6mm]
\tikzstyle{arrow}=[->,very thick]
\tikzstyle{arrowu}=[-,very thick]
\tikzstyle{arrowb}=[->,very thick,dotted]
\tikzstyle{arrowc}=[->,line width=1.0mm]
\def\distnode{2.0cm}
\def\distnodeshort{1.6cm}
\def\distnodeshortb{1.8cm}
\begin{figure}[t]
\centering
\subfloat[Sparsity pattern in an undirected graph.]{ \label{subfig:undirected}
 \begin{tikzpicture}[node distance=\distnode,>=stealth',bend angle=45,auto]
 \begin{scope}
     \node [groupwhite]    (v1)                       {$1$};
     \node [groupgray]    (v2)  [right of=v1,yshift=8mm,xshift=-2mm]        {$3$};
     \node [groupgray]    (v3)  [right of=v2,yshift=-8mm,xshift=-8mm]                     {$6$};
     \node [groupgray]    (v4)  [left of=v3,yshift=-8mm,xshift=0mm]  {$2$};
     \node [groupwhite]    (v5)  [below of=v1,yshift=8mm,xshift=-2mm]  {$4$};
     \node [groupwhite]    (v6)  [below of=v5,yshift=2mm,xshift=8mm]  {$5$};
     \node [groupwhite]    (v7)  [right of=v6,yshift=2mm,xshift=-2mm]  {$7$};
     \node [groupwhite]    (v8)  [right of=v7,yshift=10mm,xshift=-10mm]  {$8$};
     \node [groupgray]    (v9)  [left of=v6,yshift=20mm,xshift=-4mm]  {$9$};
     \node [groupwhite]    (v10)  [above of=v9,xshift=2mm,yshift=-1mm]  {$10$};
     \node [groupgray]    (v11)  [left of=v10,yshift=-10mm,xshift=5mm]  {$11$};
     \node [groupgray]    (v12)  [below of=v11,yshift=5mm,xshift=1mm]  {$12$};
     \node [groupwhite]    (v13)  [below of=v9,yshift=5mm,xshift=6mm]  {$13$};
     \draw[arrowu]        (v1) -- (v10);
     \draw[arrowu]        (v2) -- (v10);
     \draw[arrowu]        (v9) -- (v12);
     \draw[arrowu]        (v9) -- (v13);
     \draw[arrowu]        (v9) -- (v10);
     \draw[arrowu]        (v9) -- (v11);
     \draw[arrowu]        (v1) -- (v9);
     \draw[arrowu]        (v10) -- (v11);
     \draw[arrowu]        (v11) -- (v12);
     \draw[arrowu]        (v12) -- (v13);
     \draw[arrowu]        (v6) -- (v13);
     \draw[arrowu]        (v5) -- (v9);
     \draw[arrowu]        (v1) -- (v2);
     \draw[arrowu]        (v2) -- (v3);
     \draw[arrowu]        (v4) -- (v3);
     \draw[arrowu]        (v4) -- (v7);
     \draw[arrowu]        (v1) -- (v4);
     \draw[arrowu]        (v4) -- (v2);
     \draw[arrowu]        (v1) -- (v5);
     \draw[arrowu]        (v5) -- (v6);
     \draw[arrowu]        (v6) -- (v7);
     \draw[arrowu]        (v7) -- (v8);
     \draw[arrowu]        (v4) -- (v8);
     \draw[arrowu]        (v3) -- (v8);
     \draw[arrowu]        (v6) -- (v3);
     \draw[arrowu]        (v4) -- (v5);
     \end{scope}
 \end{tikzpicture}

} \hfill
\subfloat[Selected paths in a DAG.]{\label{subfig:dag}
 \begin{tikzpicture}[node distance=\distnode,>=stealth',bend angle=45,auto]
 \begin{scope}
     \node [groupwhite]    (v1)                       {$1$};
     \node [groupgray]    (v2)  [right of=v1,yshift=8mm,xshift=-2mm]        {$3$};
     \node [groupgray]    (v3)  [right of=v2,yshift=-8mm,xshift=-8mm]                     {$6$};
     \node [groupgray]    (v4)  [left of=v3,yshift=-8mm,xshift=0mm]  {$2$};
     \node [groupwhite]    (v5)  [below of=v1,yshift=8mm,xshift=-2mm]  {$4$};
     \node [groupwhite]    (v6)  [below of=v5,yshift=2mm,xshift=8mm]  {$5$};
     \node [groupwhite]    (v7)  [right of=v6,yshift=2mm,xshift=-2mm]  {$7$};
     \node [groupwhite]    (v8)  [right of=v7,yshift=10mm,xshift=-10mm]  {$8$};
     \node [groupgray]    (v9)  [left of=v6,yshift=20mm,xshift=-4mm]  {$9$};
     \node [groupwhite]    (v10)  [above of=v9,xshift=2mm,yshift=-1mm]  {$10$};
     \node [groupgray]    (v11)  [left of=v10,yshift=-10mm,xshift=5mm]  {$11$};
     \node [groupgray]    (v12)  [below of=v11,yshift=5mm,xshift=1mm]  {$12$};
     \node [groupwhite]    (v13)  [below of=v9,yshift=5mm,xshift=6mm]  {$13$};
     \draw[arrow]        (v1) -- (v10);
     \draw[arrow]        (v2) -- (v10);
     \draw[arrow]        (v9) -- (v12);
     \draw[arrow]        (v9) -- (v13);
     \draw[arrow]        (v9) -- (v10);
     \draw[arrowc]        (v9) -- (v11);
     \draw[arrow]        (v1) -- (v9);
     \draw[arrow]        (v10) -- (v11);
     \draw[arrowc]        (v11) -- (v12);
     \draw[arrow]        (v12) -- (v13);
     \draw[arrow]        (v6) -- (v13);
     \draw[arrow]        (v5) -- (v9);
     \draw[arrow]        (v1) -- (v2);
     \draw[arrowc]        (v2) -- (v3);
     \draw[arrow]        (v4) -- (v3);
     \draw[arrow]        (v4) -- (v7);
     \draw[arrow]        (v1) -- (v4);
     \draw[arrowc]        (v4) -- (v2);
     \draw[arrow]        (v1) -- (v5);
     \draw[arrow]        (v5) -- (v6);
     \draw[arrow]        (v6) -- (v7);
     \draw[arrow]        (v7) -- (v8);
     \draw[arrow]        (v4) -- (v8);
     \draw[arrow]        (v3) -- (v8);
     \draw[arrow]        (v6) -- (v3);
     \draw[arrow]        (v4) -- (v5);
     \end{scope}
 \end{tikzpicture} 

}
\caption{ Left \subref{subfig:undirected}: an undirected graph. A sparsity pattern forming a subgraph with two connected components is represented by gray nodes. Right \subref{subfig:dag}: when the graph is a DAG, the sparsity pattern is covered by two paths~$(2,3,6)$ and~$(9,11,12)$ represented by bold arrows.}
\label{fig:graphsparsity}
\end{figure}
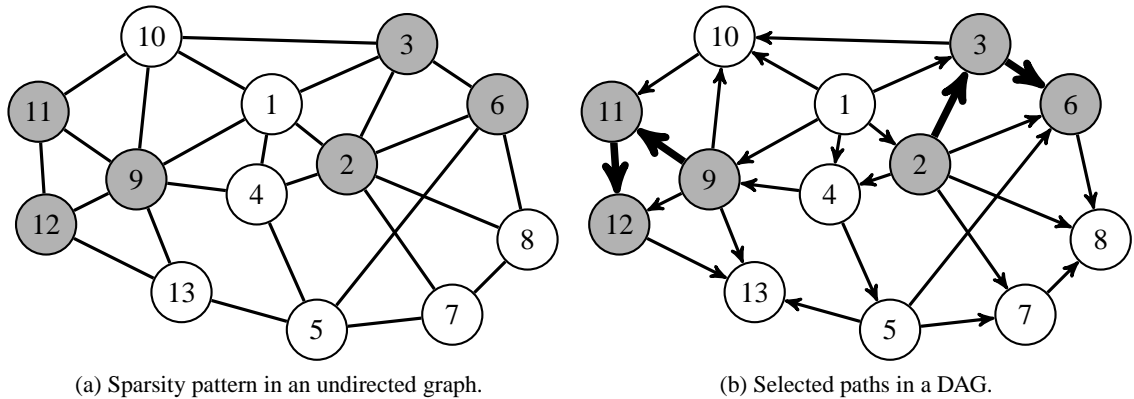

Therefore, we make in this paper a new link between structured graph penalties
in DAGs and network flow optimization.  The development of
network flow optimization techniques has been very active from the 60's to the
90's~\citep[see][]{ford,goldberg,ahuja,goldberg2,bertsekas2}.  They have
attracted a lot of attention during the last decade in the computer vision
community for their ability to solve large-scale combinatorial problems
typically arising in image segmentation tasks~\citep{boykov}. Concretely,
by mapping a problem at hand to a network flow formulation, one can possibly
obtain fast algorithms to solve the original problem. Of course, such a mapping
does not always exist or can be difficult to find. This is made possible in the
context of path coding penalties thanks to decomposability properties of
the path costs, which we make explicit in Section~\ref{sec:approach}.

We remark that different network flow formulations have also been used recently
for sparse estimation~\citep{cehver,chambolle,hoefling,mairal11}.
\citet{cehver} combine for example sparsity and Markov random fields for signal
reconstruction tasks.  They introduce a non-convex penalty consisting of
pairwise interaction terms between vertices of a graph, and their approach
requires iteratively solving maximum flow problems. It has also been shown by~\citet{chambolle}
and~\citet{hoefling} that for the anisotropic total-variation
penalty, called ``fused lasso'' in statistics, the solution to
problem~(\ref{eq:prob}) can be obtained
by solving a sequence of parametric maximum flow problems. The total-variation penalty
can be useful to obtain piecewise constant solutions on a
graph~\citep[see][]{chen3}.  Finally, 
\citet{mairal11} have shown that the structured
sparsity-inducing regularization function of~\citet{jenatton} is related to network
flows in a similar way as the total variation penalty. Note that 
both~\citet{jacob} and~\citet{jenatton} use the same terminology of
``group Lasso with overlapping groups'', leading to some confusion in the literature. Yet, their works are significantly different
and are in fact complementary:
given a group structure~$\GG$, the penalty of~\citet{jacob} encourages
solutions whose sparsity pattern is a \emph{union} of a few groups, whereas the
penalty of~\citet{jenatton} promotes an
\emph{intersection} of groups. It is natural to use the framework of~\citet{jacob}
to encourage connectivity of a problem solution in a graph, e.g., by choosing~$\GG$ as the pairs of 
vertices linked by arc. It is however not obvious how to obtain this effect with
the penalty of~\citet{jenatton}. We discuss this question in more details in
Appendix~\ref{appendix:rodolphe}.

To summarize, we have designed non-convex and convex penalty
functions to do feature selection in directed acyclic graphs.  Because our
penalties involve an exponential number of variables, one for every path in the
graph, existing optimization techniques cannot be used. To deal with this
issue, we introduce network flow optimization tools that implicitly handle
the exponential number of paths, allowing the penalties and their proximal
operators to be computed in polynomial time. As a result, our penalties can
model long-range interactions in the graph and are tractable.

The paper is organized as follows: In Section \ref{sec:related}, we present
preliminary tools, notably a brief introduction to network flows.
In Section~\ref{sec:approach}, we propose the path coding penalties and
optimization techniques for solving the corresponding sparse
estimation problems.  Section~\ref{sec:exp} is devoted to experiments on
synthetic, genomic, and image data to demonstrate the benefits of path coding penalties
over existing ones and the scalability of our approach. 
Section~\ref{sec:ccl} concludes the paper.

%% file: related.tex
As we show later, our path coding penalties are related to the
concept of flow in a graph.  Since this concept is not widely used in the
machine learning literature, we provide a brief overview of this topic in
Section~\ref{subsec:flowreview}. In Section~\ref{subsec:proxgrad}, we also
present proximal gradient methods, which have become very popular for solving sparse
regularized problems~\citep[see][]{bach8}.

\subsection{Network Flow Optimization}\label{subsec:flowreview}
Network flows have been well studied in the computer science community, and
have led to efficient dedicated algorithms for solving particular linear
programs~\citep[see][]{ahuja,bertsekas2}.  Let us consider a directed graph
$G\!=\!(V,E)$ with two special nodes~$s$ and~$t$, respectively
dubbed \emph{source} and~\emph{sink}.
A \emph{flow}
$f$ on the graph~$G$~is defined as a non-negative function on
arcs~$[f_{uv}]_{(u,v) \in E}$  that satisfies two sets of linear constraints:
\begin{myitemize}
  \item {\bfseries capacity constraints}:  the value of the flow~$f_{uv}$ on an arc~$(u,v)$ in~$E$ should satisfy the constraint
  $l_{uv}\! \leq \!f_{uv} \!\leq \!\delta_{uv}$,
  where~$l_{uv}$ and~$\delta_{uv}$ are respectively called lower and upper
  capacities; 
\item {\bfseries conservation constraints}:  the sum of incoming flow at a vertex is equal to the sum of outgoing
flow except for the source~$s$ and the sink~$t$.
\end{myitemize}
 We present examples of flows in Figures~\ref{subfig:flow1}
and~\ref{subfig:flow2}, and denote by~$\FF$ the set of flows on a graph~$G$. 
We remark that with appropriate graph
transformations, the flow definition we have given admits several variants.
It is indeed possible to
consider several source and sink nodes, capacity constraints on the amount
of flow going through vertices, or several arcs with different capacities
between two vertices~\citep[see more details in][]{ahuja}.  

\tikzstyle{arrowd}=[->,line width=0.8mm,red]
\begin{figure}[t]
\centering
\subfloat[A flow in a DAG.]{ \label{subfig:flow1}
       \begin{tikzpicture}[node distance=\distnodeshortb,>=stealth',bend angle=45,auto]
     \begin{scope}
    \node [group]   (v1)                       {1};
    \node    (v1b) [below of=v1,yshift=6mm]                      {~};
    \node [group]  (v2)  [below of=v1b,yshift=6mm]        {2};
    \node [group]   (v3)  [right of=v1b,xshift=-4mm]                     {3};
    \node [source]   (s)  [left of=v1b,xshift=4mm]                     {$s$};
    \node [group]   (v4)  [right of=v3]  {4};
    \node [sink]   (t)  [right of=v4]                     {$t$};
    \draw[arrow]        (v1) -- node {$1$} (v3);
    \draw[arrow]        (v2) -- node[below] {$1$} (v3);
    \draw[arrow]        (v3) -- node {$2$} (v4);
    \draw[arrow]        (s) -- node {$1$} (v1);
    \draw[arrow]        (s) -- node[below] {$1$} (v2);
    \draw[arrow]        (v4) -- node {$2$} (t);
    \end{scope}
    \end{tikzpicture}
} \hfill
\subfloat[A flow in a directed graph with a cycle.]{\label{subfig:flow2}
    \begin{tikzpicture}[node distance=\distnodeshortb,>=stealth',bend angle=0,auto]
     \begin{scope}
    \node [group]   (v1)                       {1};
    \node    (v1b) [below of=v1,yshift=6mm]                      {~};
    \node [group]  (v2)  [below of=v1b,yshift=6mm]        {2};
    \node [group]   (v3)  [right of=v1b,xshift=-4mm]                     {3};
    \node [source]   (s)  [left of=v1b,xshift=4mm]                     {$s$};
    \node [group]   (v4)  [right of=v3]  {4};
    \node [sink]   (t)  [right of=v4]                     {$t$};
    \draw[arrow]        (v1) -- node {$2$} (v3);
    \draw[arrow]        (v2)  -- node {$1$} (v1);
    \draw[arrow]        (v3) -- node {$2$} (v4);
    \draw[arrow]        ([yshift=-6mm,xshift=0mm] v3) -- node[above,yshift=2mm] {$1$} ([yshift=-6mm,xshift=0mm] v2);
    \draw[arrow]        ([yshift=6mm] v2) -- node[below,yshift=-2mm] {$1$} ([yshift=6mm] v3);
    \draw[arrow]        (s) -- node {$1$} (v1);
    \draw[arrow]        (s) -- node[below] {$1$} (v2);
    \draw[arrow]        (v4) -- node {$2$} (t);
    \end{scope}
    \end{tikzpicture}
} \\
\subfloat[$(s,t)$-path flow in a DAG.]{ \label{subfig:flow3}
       \begin{tikzpicture}[node distance=\distnodeshortb,>=stealth',bend angle=45,auto]
     \begin{scope}
    \node [group]   (v1)                       {1};
    \node    (v1b) [below of=v1,yshift=6mm]                      {~};
    \node [group]  (v2)  [below of=v1b,yshift=6mm]        {2};
    \node [group]   (v3)  [right of=v1b,xshift=-4mm]                     {3};
    \node [source]   (s)  [left of=v1b,xshift=4mm]                     {$s$};
    \node [group]   (v4)  [right of=v3]  {4};
    \node [sink]   (t)  [right of=v4]                     {$t$};
    \draw[arrow]        (v1) -- node {~} (v3);
    \draw[arrowd]        (v2) -- node[below] {$1$} (v3);
    \draw[arrowd]        (v3) -- node {$1$} (v4);
    \draw[arrow]        (s) -- node {~} (v1);
    \draw[arrowd]        (s) -- node[below] {$1$} (v2);
    \draw[arrowd]        (v4) -- node {$1$} (t);
    \end{scope}
    \end{tikzpicture}
} \hfill
\subfloat[A cycle flow in a directed graph.]{\label{subfig:flow4}
    \begin{tikzpicture}[node distance=\distnodeshortb,>=stealth',bend angle=0,auto]
     \begin{scope}
    \node [group]   (v1)                       {1};
    \node    (v1b) [below of=v1,yshift=6mm]                      {~};
    \node [group]  (v2)  [below of=v1b,yshift=6mm]        {2};
    \node [group]   (v3)  [right of=v1b,xshift=-4mm]                     {3};
    \node [source]   (s)  [left of=v1b,xshift=4mm]                     {$s$};
    \node [group]   (v4)  [right of=v3]  {4};
    \node [sink]   (t)  [right of=v4]                     {$t$};
    \draw[arrowd]        (v1) -- node {$1$} (v3);
    \draw[arrowd]        (v2)  -- node {$1$} (v1);
    \draw[arrow]        (v3) -- node {~} (v4);
    \draw[arrowd]        ([yshift=-6mm,xshift=0mm] v3) -- node[above,yshift=2mm] {~} ([yshift=-6mm,xshift=0mm] v2);
    \draw[arrow]        ([yshift=6mm] v2) -- node[below,yshift=-2mm,xshift=2mm] {\color{red}$1$} ([yshift=6mm] v3);
    \draw[arrow]        (s) -- node {~} (v1);
    \draw[arrow]        (s) -- node[below] {~} (v2);
    \draw[arrow]        (v4) -- node {~} (t);
    \end{scope}
    \end{tikzpicture}
} 
\caption{Examples of flows in a graph. \subref{subfig:flow1} The flow on the DAG can be interpreted as two units of flow sent from~$s$ to~$t$ along the paths $(s,1,3,4,t)$ and~$(s,2,3,4,t)$. \subref{subfig:flow2} The flow can be interpreted as two units of flow sent from~$s$ to~$t$ on the same paths as in~\subref{subfig:flow1} plus a unit of flow circulating along the cycle~$(1,3,2,1)$. \subref{subfig:flow3} $(s,t)$-path flow along the path~$(s,2,3,4,t)$. \subref{subfig:flow4} Cycle flow along~$(1,3,2,1)$.}\label{fig:flows}
\end{figure}
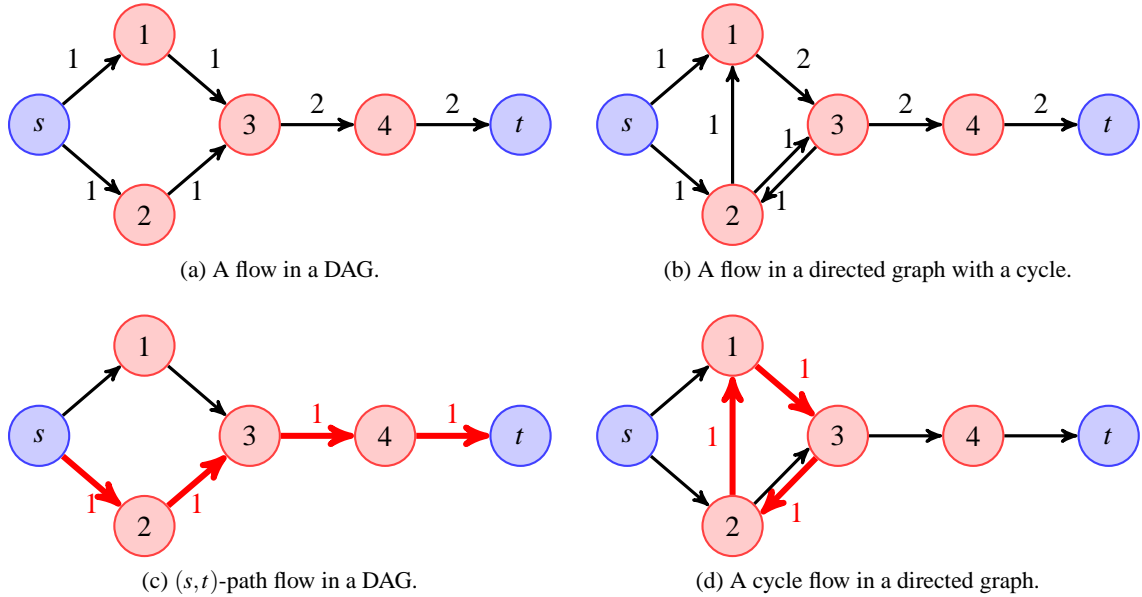

Some network flow problems have attracted a lot of attention because of their
wide range of applications, for example in engineering, physics, transportation,
or telecommunications~\citep[see][]{ahuja}. In particular, the~\emph{maximum
flow problem} consists of computing how much
flow can be sent from the source to the sink through the network~\citep{ford}. In other words, it consists of finding a flow~$f$ in~$\FF$ maximizing $\sum_{u \in V: (s,u) \in E} f_{su}$.
Another more general problem, which is of interest to us, is the \emph{minimum
cost flow} problem. It consists of finding a flow~$f$
in~$\FF$ minimizing a linear cost $\sum_{(u,v) \in E} c_{uv} f_{uv}$,
where every arc $(u,v)$ in~$E$ has a cost~$c_{uv}$ in~$\Real$.
Both the maximum flow and minimum cost flow problems are linear programs, and 
can therefore be solved using generic linear programming tools, e.g.,
interior points methods~\citep[see][]{boyd,nocedal}.
Dedicated algorithms exploiting the network structure
of flows have however proven to be much more efficient.
It has indeed been shown that minimum cost flow problems can be solved in
strongly polynomial time---that is, an exact solution can be obtained in a
finite number of steps that is polynomial in~$|V|$
and~$|E|$~\citep[see][]{ahuja}.  More important, these dedicated
algorithms are empirically efficient and can often handle large-scale
problems~\citep{goldberg,goldberg2,boykov}. 

Among linear programs, network flow problems have a few distinctive features. The most
striking one is the ``physical'' interpretation of a flow as a sum of quantities
circulating in the network. The
\emph{flow decomposition theorem} \citep[see][Theorem 3.5]{ahuja} makes this interpretation more precise
by saying that every flow vector can always be decomposed into a sum of $(s,t)$-path
flows (units of flow sent from~$s$ to~$t$ along a path) and cycle flows (units
of flow circulating along a cycle in the graph). We give examples of
$(s,t)$-path and cycle flow in Figures~\ref{subfig:flow3} and~\ref{subfig:flow4},
and present examples of flows in Figures~\ref{subfig:flow1}
and~\ref{subfig:flow2} along with their decompositions. Built upon the interpretation of flows as quantities circulating
in the network, efficient algorithms have been developed, e.g., the classical
\emph{augmenting path} algorithm
of~\citet{ford} for solving maximum flow problems.  Another feature of flow
problems is the locality of the constraints;
each one only involves neighbors of a vertex in the graph.
This locality is also exploited to design  
algorithms~\citep{goldberg,goldberg2}. Finally, minimum cost flow problems have
a remarkable \emph{integrality property}: a minimum cost flow problem where all
capacity constraints are integers can be shown to have an integral
solution~\citep[see][]{ahuja}.

Later in our paper, we will map path selection problems to network flows by
exploiting the flow decomposition theorem. In a nutshell, this apparently
simple theorem has an interesting consequence: minimum cost flow
problems can be seen from two equivalent viewpoints. Either one is looking for
the value $f_{uv}$ of a flow on every arc $(u,v)$ of a graph minimizing the
cost $\sum_{(u,v) \in E} c_{uv} f_{uv}$, or one is looking for the quantity
of flow that should circulate on every $(s,t)$-path and cycle
flow for minimizing the same cost. Of course, when the graph $G$ is a DAG, cycle flows do not
exist. We will define flow problems such that selecting a path in the context of our
path coding penalties is equivalent to sending some flow along a corresponding
$(s,t)$-path. We will also exploit the \emph{integrality property} to develop
tools both adapted to non-convex penalties and convex ones, respectively
involving discrete and continuous optimization problems.  With these tools in
hand, we will be able to deal efficiently with a simple class of optimization
problems involving our path coding penalties. To deal with the more complex
problem~(\ref{eq:prob}), we will need additional tools, which we now present.

\subsection{Proximal Gradient Methods}\label{subsec:proxgrad}
Proximal gradient methods are iterative schemes for minimizing
objective functions of the same form as~(\ref{eq:prob}), when
the function~$L$ is convex and differentiable
 with a Lipschitz continuous gradient.
The simplest proximal gradient method consists of linearizing at each
iteration the function $L$ around a current estimate $\tildew$, and this
estimate is updated as the (unique by strong convexity) solution to
\begin{equation}
    \min_{\w \in \R{p}} \Big[\underbrace{L(\tildew) + \nabla L(\tildew)^\top (\w - \tildew)}_{\text{linear approximation of } L}  + \underbrace{\frac{\rho}{2}\|\w - \tildew\|_2^2}_{\text{quadratic term}} + \underbrace{\lambda \Omega(\w)}_{\text{non-smooth part}}\Big],\label{eq:prox1}
\end{equation}
which is assumed to be easier to solve than the original
problem~(\ref{eq:prob}).  The quadratic term keeps the update in a
neighborhood where $L$ is close to its linear approximation, and the parameter $\rho$ is
an upper bound on the Lipschitz constant of $\nabla
L$.  When~$\Omega$ is convex, this scheme is known to converge to the
solution to problem~(\ref{eq:prob}) and admits variants with optimal
convergence rates among first-order methods~\citep{nesterov,beck}.
When~$\Omega$ is non-convex, the guarantees are weak (finding the global
optimum is out of reach), but it is easy to show that these updates
can be seen as a majorization-minimization
algorithm~\citep[][]{hunger} iteratively decreasing the value of the
objective function~\citep{wright,mairal14}. When~$\Omega$ is the~$\ell_1$- or ~$\ell_0$-penalty,
the optimization schemes~(\ref{eq:prox1}) are respectively known
as iterative soft- and hard-thresholding algorithms~\citep{daubechies,blumensath}.
Note that when $L$ is not differentiable, similar schemes exist, known
as mirror-descent~\citep{nemirovsky}.

Another insight about these methods can be obtained by rewriting sub-problem~(\ref{eq:prox1}) as
\begin{displaymath}
   \min_{\w \in \R{p}}\left[ {\displaystyle \frac{1}{2}} \Big\|\tildew - \frac{1}{\rho}
\nabla L(\tildew)-\w\Big\|_2^2 + \frac{\lambda}{\rho} \Omega(\w)\right].
\end{displaymath}
When~$\lambda=0$, the solution is obtained by a classical gradient
step~$\tildew \leftarrow \tildew - (1/\rho)\nabla L(\tildew)$.
Thus, proximal gradient methods can be interpreted as a generalization of
gradient descent algorithms when dealing with a nonsmooth term.
They are, however, only interesting when problem~(\ref{eq:prox1}) can be efficiently solved.
Formally, we wish to be able to compute the \emph{proximal operator} defined as:
\begin{definition}[Proximal Operator.]\label{definition:prox}~\newline
The proximal operator associated with a regularization term $\lambda\Omega$, which we denote by $\text{Prox}_{\lambda \Omega}$, is the function that maps a vector $\u \in \Real^{p}$ to the unique (by strong convexity) solution to
\begin{equation}\label{eq:prox_problem}
   \min_{\w \in \Real^{p}} \left[\frac{1}{2} \|\u-\w\|_2^2 + \lambda  \Omega(\w)\right].
\end{equation}
\end{definition}
Computing efficiently this operator has been shown to be possible for many
penalties~$\Omega$~\citep[see][]{bach8}.  We will show in the sequel that it is
also possible for our path coding penalties.

%% file: approach.tex
We now present our path coding penalties, which exploit the structured sparsity
frameworks of~\citet{jacob} and~\citet{huang}. Because we choose a group structure~$\G_p$ with an exponential
number of groups, one for every path in the graph, the optimization techniques
presented by~\citet{jacob} or~\citet{huang} cannot be used anymore. We will
deal with this issue by introducing flow definitions of the path coding penalties.

\subsection{Path Coding Penalties}\label{subsec:penalties}
The so-called ``block coding'' penalty of~\citet{huang} can be written for a
vector~$\w$ in~$\Real^p$ and any set~$\G$ of groups of variables as
\begin{equation} 
\varphi(\w) \defin \min_{\JJ \subseteq \GG} \Big\{ \sum_{g
\in \JJ} \eta_{g}  \st \text{Supp}(\w) \subseteq \bigcup_{g \in \JJ} g
\Big\}, \label{eq:nonconvex} 
\end{equation} 
where the $\eta_g$'s are non-negative weights, and $\JJ$ is a subset of
groups in $\GG$ whose union covers the support of~$\w$ formally defined as $\text{Supp}(\w) \defin \{ j \in \{1,\ldots,p\} : \w_j \neq 0 \}$.
When the
weights~$\eta_g$ are well chosen, the non-convex penalty $\varphi$ encourages solutions~$\w$
whose support is in the union of a small number of groups;
in other words, the cardinality of~$\JJ$ should be small.
We remark that~\citet{huang} originally introduce this regularization function
under a more general information-theoretic point of view where $\varphi$ is a code
length~\citep[see][]{barron,cover}, and the weights~$\eta_g$ represent the number
of bits encoding the fact that a group~$g$ is selected.\footnote{Note
that~\citet{huang} do not directly use the function~$\varphi$ as a
regularization function. The ``coding complexity'' they introduce for a vector~$\w$
counts the number of bits to code the support of~$\w$, which is achieved
by~$\varphi$, but also use an $\ell_0$-penalty to count the number of bits
encoding the values of the non-zero coefficients in~$\w$.} 
One motivation for using $\varphi$ is that the selection of a few groups
might be easier to interpret than the selection of isolated variables. This
formulation extends non-convex group sparsity
regularization by allowing any group structure~$\GG$ to be considered.
Nevertheless, a major drawback is that computing this non-convex
penalty~$\varphi(\w)$ for a general group structure~$\G$ is difficult.
Equation~(\ref{eq:nonconvex}) is indeed an instance of a set cover problem,
which is NP-hard~\citep[see][]{cormen}, and appropriate
approximations, e.g., greedy algorithms, have to be used in practice. 

As often when dealing with non-convex penalties, one can either try to solve
directly the corresponding non-convex problems or look for a convex
relaxation. As we empirically show in Section~\ref{sec:exp}, having both
non-convex and convex variants of a penalty can be a significant asset. One
variant can indeed outperform the other one in some situations, while being the
other way around in some other cases.
It is therefore interesting to look for a convex relaxation of~$\varphi$. We
denote by~$\eta$ the vector $[\eta_g]_{g\in \GG}$ in~$\Real_+^{|\GG|}$,
and by~$\NN$ the binary matrix in $\{0,1\}^{p \times |\GG|}$ whose columns are
indexed by the groups $g$ in $\GG$, such that the entry~$\NN_{jg}$ is equal to
one when the index~$j$ is in the group $g$, and zero otherwise. 
Equation~(\ref{eq:nonconvex}) can be rewritten as a Boolean linear program,
a form which will be more convenient in the rest of the paper:
\begin{equation}
\varphi(\w) = \min_{  \x \in \{0,1\}^{|\GG|}} \left\{ \eta^\top \x  \st \NN\x
\geq \text{Supp}(\w) \right\}, \label{eq:altphi}
\end{equation} 
where, with an abuse of notation, $\text{Supp}(\w)$ is here a vector in
$\{0,1\}^p$ such that its $j$-th entry is $1$ if~$j$ is in the support of $\w$
and $0$ otherwise.  Let us also denote by~$|\w|$
the vector in~$\Real_+^p$ obtained by replacing the entries of~$\w$ by their absolute value.
We can now consider a convex relaxation of~$\varphi$:
\begin{equation}
\psi(\w) \defin \min_{  \x \in \Real_+^{|\GG|}} \left\{ \eta^\top \x  \st
\NN\x \geq |\w|  \right\}, \label{eq:convex} 
\end{equation} 
where not only the optimization problem above is a linear program, but in
addition $\psi$ is a convex function---in fact it can be shown to be a norm.
Such a relaxation is classical and corresponds to the same mechanism relating
the~$\ell_0$- to the~$\ell_1$-penalty, replacing $\text{Supp}(\w)$ by~$|\w|$.
The next lemma tells us that we have in fact obtained a variant of the penalty
introduced by~\citet{jacob}.
\begin{lemma}[Relation Between~$\psi$ and the Penalty of~\citet{jacob}.]\label{lemma:equiv}~\newline 
Suppose that any pattern in~$\{0,1\}^p$ can be represented by a union of groups
in~$\GG$. Then, the function~$\psi$ defined in~(\ref{eq:convex}) is equal to the penalty of~\citet{jacob} with~$\ell_\infty$-norms.
\end{lemma} 
Note that~\citet{jacob} have introduced their penalty from a
different perspective, and the link between~(\ref{eq:convex}) and their work
is not obvious at first sight. In addition, their penalty involves a sum
of~$\ell_2$-norms, which needs to be replaced by~$\ell_\infty$-norms for the
lemma to hold. Hence, $\psi$ is a ``variant'' of the penalty of~\citet{jacob}. 
We give more details and the proof
of this lemma in Appendix~\ref{appendix:links}.\footnote{At
the same time as us, \citet{obozinski3} have studied a larger class of
non-convex combinatorial penalties and their corresponding convex relaxations, obtaining
 in particular a more general result than Lemma~\ref{lemma:equiv}, showing that~$\psi$ is the tightest convex relaxation of $\varphi$.}

Now that~$\varphi$ and~$\psi$ have been introduced, we are interested in
automatically selecting a small number of connected subgraphs from a directed
acyclic graph~$G=(V,E)$.  In Section~\ref{sec:intro}, we already discussed
group structures~$\G$ and introduced $\G_p$ the
\emph{set of paths in~$G$}. As a result, the path coding penalties~$\varphip$
and~$\psip$ encourage solutions that are sparse while forming a subgraph that
can be covered by a small number of paths. As we show in this section, this choice leads
to tractable formulations when the weights~$\eta_g$ for every path~$g$ in~$\G_p$ are appropriately chosen.

We will show in the sequel that a natural choice is to define
for all~$g$ in~$\G_p$ 
\begin{equation}
\eta_g\defin\gamma + |g|,\label{eq:choice}
\end{equation}
where~$\gamma$ is a new parameter encouraging the connectivity of the solution
whereas~$|g|$ encourages sparsity.  It is indeed possible to show that
when~$\gamma\!=\!0$, the functions~$\varphip$ and~$\psip$ respectively become
the~$\ell_0$- and the $\ell_1$-penalties, therefore encouraging sparsity but
not connectivity. On the other hand, when~$\gamma$ is large and the term~$|g|$
is negligible,~$\varphip(\w)$ simply ``counts'' how many paths are required to
cover the support of~$\w$, thereby encouraging connectivity regardless of the
sparsity of~$\w$. 

In fact, the choice~(\ref{eq:choice}) is a particular case of a more general
class of weights~$\eta_g$, which our algorithmic framework can handle. Let us enrich
the original directed acyclic graph~$G$ by introducing a source node~$s$ and a
sink node~$t$. Formally, we define a new graph $G'=(V',E')$ with
\begin{displaymath}
\begin{split}
V' & \defin V \cup \{s,t\},\\
E' &\defin E \cup \{ (s,v) : v \in V\} \cup  \{ (u,t) : u \in V\}.
\end{split}
\end{displaymath}
In plain words, the graph $G'$, which is a DAG, contains the graph $G$ and two
nodes $s,t$ that are linked to every vertices of $G$. Let us also assume
that some costs $c_{uv}$ in~$\Real$ are defined for all arcs~$(u,v)$
in~$E'$. Then, for a path~$g=(u_1,u_2,\ldots,u_k)$ in~$\G_p$, we define the weight~$\eta_g$ as
\begin{equation}
   \eta_g \defin c_{su_1} + \Big(\sum_{i=1}^{k-1} c_{u_{i}u_{i+1}}\Big) + c_{u_k t} = \sum_{(u,v) \in (s,g,t)} c_{uv},\label{eq:choice2}
\end{equation}
where the notation $(s,g,t)$ stands for the path $(s,u_1,u_2,\ldots,u_k,t)$ in $G'$.
The decomposition of the weights~$\eta_g$ as a sum of costs on $(s,t)$-paths of~$G'$ (the paths $(s,g,t)$ with $g$ in~$\G_p$) is 
a key component of the algorithmic framework we present next.
The construction of the graph $G'$ is illustrated in Figures~\ref{subfig:grapha}
and~\ref{subfig:graphb} for two cost configurations.
We remark that the simple choice of weights~(\ref{eq:choice}) corresponds to the
choice~(\ref{eq:choice2}) with the costs $c_{su}=\gamma$ for all~$u$ in~$V$ and
$c_{uv}=1$ otherwise (see Figure~\ref{subfig:grapha}).
Designing costs~$c_{uv}$ that go beyond the simple choice~(\ref{eq:choice})
can be useful whenever one has additional knowledge about the graph structure. 
For example, we experimentally exploit this property in Section~\ref{subsec:image} to
privilege or penalize paths~$g$ in~$\GG_p$ starting from a particular vertex. This is
illustrated in Figure~\ref{subfig:graphb} where the cost on the arc~$(s,1)$ is much smaller
than on the arcs $(s,2)$, $(s,3)$, $(s,4)$, therefore encouraging paths starting from vertex~$1$.

Another interpretation connecting the path-coding penalties with coding lengths
and random walks can be drawn using information theoretic arguments derived
from~\citet{huang}. We find these connections interesting, but for simplicity only
present them in Appendix~\ref{appendix:coding}.  In the next sections, we address the following
issues: (i) how to compute the penalties~$\varphip$ and~$\psip$ given a
vector~$\w$ in~$\Real^p$? (ii) how to optimize the objective function~(\ref{eq:prob})?  (iii) in the convex case (when~$\Omega=\psip$),
can we obtain practical optimality guarantees via a duality gap?  All of these
questions will be answered using network flow and convex optimization, or
algorithmic tools on graphs.

 \tikzstyle{source}=[circle,thick,draw=blue!75,fill=blue!20,minimum size=8mm]
 \tikzstyle{sink}=[circle,thick,draw=blue!75,fill=blue!20,minimum size=8mm]
 \tikzstyle{group}=[place,thick,draw=red!75,fill=red!20, minimum size=8mm]
 \tikzstyle{var}=[place,thick,draw=red!75,fill=red!20, minimum size=8mm]
 \tikzstyle{arrow}=[->,very thick]
 \tikzstyle{arrowb}=[->,very thick,dotted]
 \tikzstyle{arrowc}=[->,line width=0.8mm,red]
 \def\distnode{2.6cm}
 \def\distnodeshort{2.3cm}
 \def\distnodeshortb{2.6cm}
 \begin{figure}[hbtp]
 \tikzstyle{every label}=[red]
 \subfloat[Graph $G'$ with arcs costs and a path~$g$ in bold red.]{ \label{subfig:grapha}
\begin{tikzpicture}[node distance=\distnode,>=stealth',bend angle=45,auto]
\begin{scope}
\node [group]    (v1)                       {$1$};
\node [group]    (v2)  [right of=v1,yshift=6mm,xshift=4mm]        {$2$};
\node [group]    (v3)  [right of=v2,yshift=-6mm,xshift=0mm]                     {$3$};
\node [group]    (v4)  [left of=v3,yshift=-6mm,xshift=-4mm]  {$4$};
\node [source]    (s)  [above of=v2,xshift=-2mm,yshift=-3mm]  {s};
\node [sink]    (t)  [below of=s,yshift=-30mm]  {t};
\draw[arrow]        (v1) edge node[above=1pt] {$1$} (v2);
\draw[arrowc]        (v2) edge node[above=1pt] {$1$} (v3);
\draw[arrow]        (v4) edge node[below=1pt] {$1$} (v3);
\draw[arrow]        (v1) edge node[below=1pt] {$1$} (v4);
\draw[arrowc]        (v4) edge node[yshift=-2mm,xshift=0mm] {$1$} (v2);
\draw[arrowb]        (s) edge[in=80,out=200] node[above,yshift=0mm]{$\gamma$} (v1);
\draw[arrowb]        (s) edge[in=100,out=-20] node[yshift=-1mm,xshift=-1mm] {$\gamma$} (v3);
\draw[arrowb]        (s) edge[in=90,out=-80]  node[yshift=-2mm]{$\gamma$} (v2);
\draw[arrowc]        (s) edge[in=120,out=-110]  node[above,xshift=-0mm,yshift=4mm]{$\gamma$} (v4);
\draw[arrowb]        (v1) edge[in=170,out=-90] node[below] {$1$}(t);
\draw[arrowc]        (v3) edge[in=10,out=-90]  node {$1$}(t);
\draw[arrowb]        (v2) edge[in=70,out=-80]  node[near end,yshift=2mm] {$1$}(t);
\draw[arrowb]        (v4) edge[in=100,out=-90]  node[left,xshift=1mm] {$1$}(t);
\end{scope}
\end{tikzpicture}
} \hfill
\subfloat[Graph~$G'$ with different arcs costs and a path~$g$.]{ \label{subfig:graphb}
\begin{tikzpicture}[node distance=\distnode,>=stealth',bend angle=45,auto]
\begin{scope}
\node [group]    (v1)                       {$1$};
\node [group]    (v2)  [right of=v1,yshift=6mm,xshift=4mm]        {$2$};
\node [group]    (v3)  [right of=v2,yshift=-6mm,xshift=0mm]                     {$3$};
\node [group]    (v4)  [left of=v3,yshift=-6mm,xshift=-4mm]  {$4$};
\node [source]    (s)  [above of=v2,xshift=-2mm,yshift=-3mm]  {s};
\node [sink]    (t)  [below of=s,yshift=-30mm]  {t};
\draw[arrowc]        (v1) edge node[above=1pt] {$2$} (v2);
\draw[arrow]        (v2) edge node[above=1pt] {$2$} (v3);
\draw[arrow]        (v4) edge node[below=1pt] {$3$} (v3);
\draw[arrow]        (v1) edge node[below=1pt] {$3$} (v4);
\draw[arrowb]        (v4) edge node[yshift=-2mm,xshift=0mm] {$1$} (v2);
\draw[arrowc]        (s) edge[in=80,out=200] node[above,yshift=0mm]{$1$} (v1);
\draw[arrowb]        (s) edge[in=100,out=-20] node[yshift=-1mm,xshift=-1mm] {$5$} (v3);
\draw[arrowb]        (s) edge[in=90,out=-80]  node[yshift=-2mm]{$5$} (v2);
\draw[arrowb]        (s) edge[in=120,out=-110]  node[above,xshift=-0mm,yshift=4mm]{$5$} (v4);
\draw[arrowb]        (v1) edge[in=170,out=-90] node[below] {$1$}(t);
\draw[arrowb]        (v3) edge[in=10,out=-90]  node {$1$}(t);
\draw[arrowc]        (v2) edge[in=70,out=-80]  node[near end,yshift=2mm] {$1$}(t);
\draw[arrowb]        (v4) edge[in=100,out=-90]  node[left,xshift=1mm] {$1$}(t);
\end{scope}
\end{tikzpicture}
} 
 \caption{\subref{subfig:grapha} $G'$ is obtained by adding a source~$s$ and sink~$t$ to a DAG with four nodes. The cost configuration is such that the weights~$\eta_g$ satisfy $\eta_g=\gamma+|g|$. For example, for~$g=(4,2,3)$, the sum of costs along~$(s,g,t)$ is $\eta_g=\gamma+3$.
 \subref{subfig:graphb} Same graph $G'$ as \subref{subfig:grapha} but with different costs. The weight~$\eta_g$ associated to the path~$g=(1,2)$ is the sum of costs along $(s,1,2,t)$---that is,~$\eta_g=4$. 
 }
 \label{fig:paths}
 \end{figure}
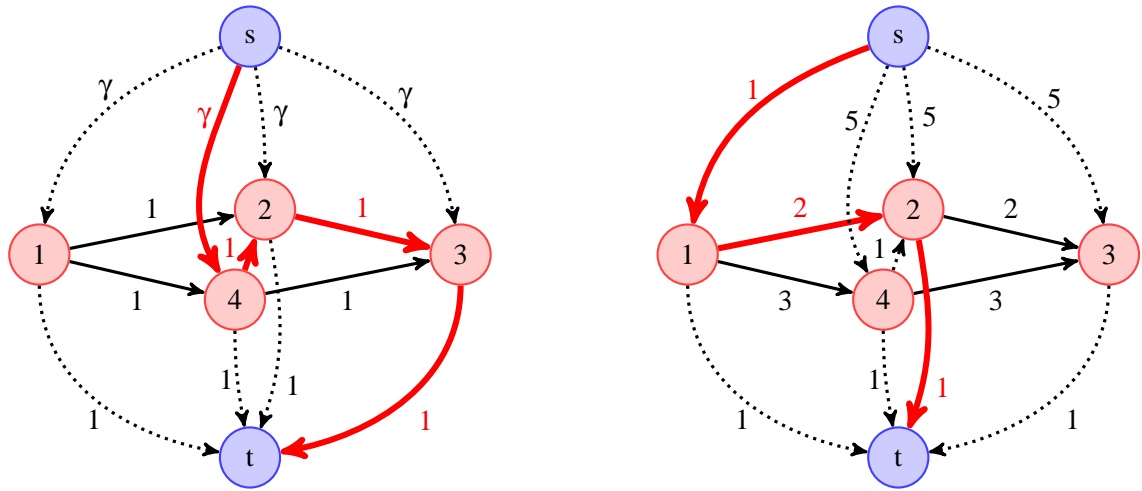

%% file: optim.tex
\subsection{Flow Definitions of the Path Coding Penalties}\label{subsec:graphopt}
Before precisely stating the flow definitions of~$\varphip$ and~$\psip$,
let us sketch the main ideas.
The first key component is to transform the
optimization problems~(\ref{eq:nonconvex}) and~(\ref{eq:convex}) over the paths
in~$G$ into optimization problems over $(s,t)$-\emph{path flows} in~$G'$. We recall that 
$(s,t)$-path flows are defined as flow vectors carrying the same positive value on every arc of a
path between $s$ and $t$. It intuitively corresponds to sending from~$s$ to~$t$
a positive amount of flow along a path, an interpretation we have presented in
Figure~\ref{fig:flows} from Section~\ref{subsec:flowreview}.
Then, we use the \emph{flow decomposition theorem}
(see Section~\ref{subsec:flowreview}),  
which provides two equivalent viewpoints for solving a minimum cost flow problem
on a DAG. One should be either looking for the value $f_{uv}$ of a flow on
every arc $(u,v)$ of the graph, or one should decide how much flow should be
sent on every $(s,t)$-path.

We assume that a cost configuration~$[c_{uv}]_{(u,v) \in E'}$ is available and
that the weights~$\eta_g$ are defined according to Equation~(\ref{eq:choice2}).
We denote by~$\FF$ the set of flows on~$G'$.  The second key component of our
approach is the fact that the cost of a flow $f$ in~$\FF$ sending one unit
from~$s$ to~$t$ along a path~$g$ in~$G$, defined as $\sum_{(u,v) \in E'}
f_{uv}c_{uv} = \sum_{(u,v) \in (s,g,t)} c_{uv}$ is exactly~$\eta_g$, according
to Equation~(\ref{eq:choice2}).  This enables us to reformulate our
optimization problems~(\ref{eq:nonconvex}) and~(\ref{eq:convex}) on paths
in~$G$ as optimization problems on~$(s,t)$-path flows in~$G'$, which in turn
are equivalent to minimum cost flow problems and can be solved in polynomial
time.  Note that this equivalence does not hold when we have cycle flows (see
Figure~\ref{subfig:flow4}), and this is the reason why we have assumed $G$ to
be acyclic. 

We can now formally state the mappings between the penalties~$\varphip$ and~$\psip$ on one hand, and network
flows on the other hand. An important quantity in the upcoming propositions is the amount of flow going
through a vertex~$j$ in~$V\!=\!\{1,\ldots,p\}$, which we denote by
$$s_j(f)\defin \sum_{u \in V' : (u,j) \in E'} f_{uj}.$$ 
Our formulations involve capacity constraints and costs for~$s_j(f)$,
which can be handled by network flow solvers; in fact, a vertex can always be
equivalently replaced in the network by two vertices, linked by an arc that
carries the flow quantity $s_j(f)$~\citep{ahuja}.
The main propositions are presented below, and the proofs are given in Appendix~\ref{appendix:proofs}.
\begin{proposition}[Computing $\varphip$.]\label{prop:varphi}~\newline 
 Let $\w$ be in $\Real^p$. Consider the network~$G'$ defined in Section~\ref{subsec:penalties} with costs $[c_{uv}]_{(u,v) \in E'}$, and define~$\eta_g$ as in~(\ref{eq:choice2}). Then,
  \begin{equation}
     \varphip(\w) =  \min_{f \in \FF} \left\{ \sum_{(u,v) \in E'} f_{uv} c_{uv} \st  s_j(f) \geq 1,~\forall j \in \text{Supp}(\w) \right\},  \label{eq:phiflow}
\end{equation}
where~$\FF$ is the set of flows on~$G'$. This is a minimum cost flow problem with some lower-capacity constraints,
which can be computed in strongly polynomial time.\footnote{See the definition of ``strongly polynomial time'' in Section~\ref{subsec:flowreview}.}
\end{proposition}
Given the definition of the penalty~$\varphi$ in Eq.~(\ref{eq:altphi}),
computing~$\varphip$ seems challenging for two reasons: (i)
Eq.~(\ref{eq:altphi}) is for a general group structure~$\G$ a NP-hard
Boolean linear program with~$|\G|$ variables; (ii) the size of $\G_p$ is
exponential in the graph size. Interestingly,
Proposition~\ref{prop:varphi} tells us that these two difficulties can be
overcome when~$\G=\G_p$ and that the non-convex penalty~$\varphip$ can be computed in
polynomial time by solving the convex optimization problem defined in
Eq.~(\ref{eq:phiflow}).
The key component
to obtain the flow definition of~$\varphip$ is the
decomposability property of the weights~$\eta_g$ defined
in~(\ref{eq:choice2}). This allows us to identify the cost of sending one unit
of flow in~$G'$ from~$s$ to~$t$ along a path~$g$ to the cost of selecting the
path~$g$ in the context of the path coding penalty~$\varphip$.
We now show that the same methodology applies to the convex penalty $\psip$.
 \begin{proposition}[Computing $\psip$.]\label{prop:psi}~\newline
 Let $\w$ be in $\Real^p$. Consider the network~$G'$ defined in Section~\ref{subsec:penalties} with costs $[c_{uv}]_{(u,v) \in E'}$, and define~$\eta_g$ as in~(\ref{eq:choice2}). Then,
  \begin{equation}
     \psip(\w) =  \min_{f \in \FF} \left\{ \sum_{(u,v) \in E'} f_{uv} c_{uv} \st  s_{j}(f) \geq |\w_j|,~\forall j \in \{1,\ldots,p\} \right\},\label{eq:psiflow}
  \end{equation}
  where~$\FF$ is the set of flows on~$G'$. This is a minimum cost flow problem with some lower-capacity constraints, which can be computed in strongly polynomial time.
\end{proposition}
From the similarity between Equations~(\ref{eq:phiflow})
and~(\ref{eq:psiflow}), it is easy to see that~$\psip$ and~$\varphip$ are
closely related, one being a convex relaxation of the other as explained
in Section~\ref{subsec:penalties}. 
We have shown here that~$\varphip$ and~$\psip$ can be computed in
polynomial time and will discuss in Section~\ref{subsec:alg} practical
algorithms to do it in practice. Before that, we address the problem of optimizing~(\ref{eq:prob}).

\subsection{Using Proximal Gradient Methods with the Path Coding Penalties}
To address the regularized problem~(\ref{eq:prob}), we use proximal gradient methods, which 
we have presented in Section~\ref{subsec:proxgrad}. We need for that to compute the proximal operators of $\varphip$ and~$\psip$ from Definition~\ref{definition:prox}.
We show that this operator can be efficiently computed by using network flow optimization.
\begin{proposition}[Computing the Proximal Operator of ${\varphip}$.]\label{prop:proxphi}~\newline
 Let $\u$ be in $\Real^p$. Consider the network~$G'$ defined in Section~\ref{subsec:penalties} with costs $[c_{uv}]_{(u,v) \in E'}$, and define~$\eta_g$ as in~(\ref{eq:choice2}).
Let us define
 \begin{equation}
    f^\star \in  \argmin_{f \in \FF} \left\{ \sum_{(u,v) \in E'} f_{uv} c_{uv} + \sum_{j=1}^p \frac{1}{2}\max\left(\u_j^2 (1-s_{j}(f)), 0\right)  \right\}, \label{eq:proxphi}
 \end{equation}
 where~$\FF$ is the set of flows on~$G'$. This is a minimum cost flow problem, with piecewise linear costs,
which can be computed in strongly polynomial time.
Denoting by $\w^\star \! \defin \! \text{Prox}_{\varphip}[\u]$, we have for all~$j$ in $V=\{1,\ldots,p\}$ that $\w^\star_j = \u_j $ if $s_f(f^\star)\! > \! 0$ and $0$ otherwise.
\end{proposition}
Note that even though the formulation~(\ref{eq:prox_problem}) is
non-convex when~$\Omega$ is the function~$\varphip$, its global optimum can
be found by solving the convex problem described in
Equation~(\ref{eq:proxphi}). 
As before, the key component to establish the mapping to a network flow problem
is the decomposability property of the weights~$\eta_g$. More details are provided
in the proofs of Appendix~\ref{appendix:proofs}.
Note also that any minimum cost flow problem with convex piecewise linear
costs can be equivalently recast as a classical minimum cost flow problem
with linear costs~\citep[see][]{ahuja}, and therefore the
above problem can be solved in strongly polynomial time.  We now present
similar results for~$\psip$.
\begin{proposition}[Computing the Proximal Operator of ${\psip}$.]\label{prop:proxpsi}~\newline
 Let $\u$ be in $\Real^p$. Consider the network~$G'$ defined in Section~\ref{subsec:penalties} with costs $[c_{uv}]_{(u,v) \in E'}$, and define~$\eta_g$ as in~(\ref{eq:choice2}).
Let us define
 \begin{equation}
    f^\star \in  \argmin_{f \in \FF} \left\{ \sum_{(u,v) \in E'} f_{uv} c_{uv} + \sum_{j=1}^p \frac{1}{2}\max\big(|\u_j| - s_j(f), 0\big)^2  \right\}, \label{eq:flowpsi}
 \end{equation}
 where~$\FF$ is the set of flows on~$G'$. This is a minimum cost flow problem, with piecewise quadratic costs, which can be computed in polynomial time.
Denoting by $\w^\star \! \defin \! \text{Prox}_{\psip}[\u]$, we have for all $j$ in $V\!=\!\{1,\ldots,p\}$, $\w^\star_j = \sign(\u_j)\min(|\u_j|,s_{j}(f^\star))$.
\end{proposition}
The proof of this proposition is presented in Appendix~\ref{appendix:proofs}.
We remark that we are dealing in Proposition~\ref{prop:proxpsi} with a minimum cost flow problem with
quadratic costs, which is more difficult to solve than
when the costs are linear.
Such problems with quadratic costs can be solved in weakly
(instead of strongly) polynomial time~\citep[see][]{hochbaum2}---that is,
a time polynomial in~$|V|$,~$|E|$ and~$\log(\|\u\|_\infty / \varepsilon)$
to obtain an $\varepsilon$-accurate solution to
problem~(\ref{eq:flowpsi}), where~$\varepsilon$ can possibly be set to
the machine precision.
We have therefore shown that the computations
of~$\varphip$,~$\psip$, $\text{Prox}_{\varphip}$ and~$\text{Prox}_{\psip}$ can
be done in polynomial time. We now discuss practical algorithms, which have
empirically shown to be efficient and scalable~\citep{goldberg2,bertsekas2}.

\subsection{Practical Algorithms for Solving the Network Flow Problems}\label{subsec:alg}
The minimum cost flow problems involved in the computations of~$\varphip$, $\psip$ and $\text{Prox}_{\varphip}$ can be solved in the 
worst-case with $O\big((|V| \log |V|)(|E| + |V|\log |V|)\big)$
operations~\citep[see][]{ahuja}.
However, this analysis corresponds to
the worst-case possible and the empirical complexity of network
flow solvers is often much better~\citep{boykov}.  Instead of a strongly
polynomial algorithm, we have chosen to implement the scaling
push-relabel algorithm~\cite{goldberg2}, also known as an $\varepsilon$-relaxation method~\citep{bertsekas2}. This algorithm is indeed empirically
efficient despite its weakly polynomial worst-case complexity.  It
requires transforming the capacities and costs of the minimum cost flow
problems into integers with an appropriate scaling and rounding
procedure, and denoting by~$C$ the (integer) value of the maximum cost
in the network its worst-case complexity is $O\big(|V|^2 |E| \log(C|V|)\big)$.  This algorithm is appealing because of its empirical
efficiency when the right heuristics are used~\cite{goldberg2}.  We
choose~$C$ to be as large as possible (using $64$ bits integers) not to
lose numerical precision, even though choosing~$C$ according to the
desired statistical precision and the robustness of the proximal gradient
algorithms would be more appropriate. It has indeed been shown recently
by~\citet{schmidt} that proximal gradient methods for convex optimization are
robust to inexact computations of the proximal operator, as long as the
precision of these computations iteratively increases with an appropriate rate.

Computing the proximal operator~$\text{Prox}_{\psip}[\u]$ requires dealing with
piecewise quadratic costs, which are more complicated to deal with than
linear costs. Fortunately, cost scaling or $\varepsilon$-relaxation techniques can
be modified to handle any convex costs, while keeping a polynomial complexity~\citep{bertsekas2}.
Concisely describing $\varepsilon$-relaxation
algorithms is difficult because their convergence properties do not come from
classical convex optimization theory. We present here an interpretation of these methods,
but we refer the reader to Chapter 9 of~\citet{bertsekas2} for more details and
implementation issues. In a nutshell, consider a primal convex cost flow
problem $\min_{f \in \FF} \sum_{(u,v) \in E} C_{uv}(f_{uv})$, where the
functions $C_{uv}$ are convex, and without capacity constraints.
Using classical Lagrangian
duality, it is possible to obtain the following dual formulation
\begin{displaymath}
\max_{ \pi \in \Real^{p}} \sum_{(u,v) \in E} q_{uv}(\pi_u - \pi_v),
~~\text{where}~~ q_{uv}(\pi_u - \pi_v) \defin \min_{ f_{uv} \geq 0}
\left[C_{uv}(f_{uv}) - (\pi_u-\pi_v)f_{uv}\right].
\end{displaymath}
This formulation is unconstrained and involves for each node $u$ in $V$ a dual
variable $\pi_u$ in~$\Real$, which is called the \emph{price} of node $u$.
$\varepsilon$-relaxation techniques rely on this dual formulation, and can be
interpreted as approximate dual coordinate ascent algorithms. They exploit the
network structure to perform computationally cheap updates of the dual and primal variables, and can deal
with the fact that the functions $q_{uv}$ are concave but not differentiable in
general. Presenting how this is achieved exactly would be too long for this paper;
we instead refer the reader to Chapter~9 of~\citet{bertsekas2}.
In the next section, we introduce algorithms to compute the dual norm of~$\psip$, which is an
important quantity to obtain optimality guarantees
for~(\ref{eq:prob}) with $\Omega=\psip$, or for implementing active set methods
that are adapted to very large-scale very sparse problems~\citep[see][]{bach8}.

\subsection{Computing the Dual-Norm of~$\psip$}
The dual norm $\psip^*$ of the norm $\psip$ is defined for any
vector $\kappab$ in $\R{p}$ as $\psip^*(\kappab)\defin\max_{\psip(\w)\leq
1}\w^\top\kappab$ \citep[see][]{boyd}. 
We show in this section that~$\psip^*$ can be computed efficiently.
\begin{proposition}[Computing the Dual Norm $\psip^*$.]\label{prop:psistar}~\newline
 Let $\kappab$ be in $\Real^p$.
 Consider the network~$G'$ defined in Section~\ref{subsec:penalties} with costs $[c_{uv}]_{(u,v) \in E'}$, and define~$\eta_g$ as in~(\ref{eq:choice2}).
For $\tau \geq 0$, and all $j$ in $\{1,\ldots,p\}$, we define
an additional cost for the vertex~$j$ to be $-|\kappab_j|/\tau$. We then define for every path~$g$ in~$\GG_p$, the length $l_\tau(g)$ to be the sum of the costs along the corresponding $(s,t)$-path from~$G'$. Then,
 \begin{displaymath}
    \psip^*(\kappab) =  \min_{\tau \in \Real_+}  \left\{ \tau \st \min_{g \in \GG_p} l_\tau(g) \geq 0 \right\}, 
 \end{displaymath}
and $\psip^*(\kappab)$ is the smallest factor $\tau$ such that the shortest $(s,t)$-path on~$G'$ has nonnegative length.
\end{proposition}
The proof is given in Appendix~\ref{appendix:proofs}. We note that
the above quantity~$l_\tau(g)$ satisfies $l_\tau(g)= \eta_g -
\|\kappab_g\|_1/\tau$, for every~$\tau>0$ and~$\kappab$ in~$\Real^p$. We 
present a simple way for computing~$\psip^*$ in
Algorithm~\ref{algo:dual_norm}, which is proven in Proposition~\ref{prop:algphi} to be correct and to converge in polynomial time.
\begin{algorithm}[!hbtp]
\caption{Computation of the Dual Norm~$\psip^*$.}\label{algo:dual_norm}
\begin{algorithmic}[1]
\INPUT $\kappab \in \R{p}$ such that~$\kappab \neq 0$.  
\STATE Choose any path $g \in \GG_p$ such that~$\kappab_g \neq 0$;
\STATE $\delta \leftarrow -\infty$;
\WHILE{ $\delta < 0$}
    \STATE $\tau \leftarrow \frac{\|\kappab_g\|_1}{\eta_g}$;
    \STATE $g \leftarrow \argmin_{h \in \GG_p} l_\tau(h)$; (shortest path problem in a directed acyclic graph);
    \STATE $\delta \leftarrow l_\tau(g)$;
\ENDWHILE
\STATE {\bf{Return:}} $\tau=\psip^*(\kappab)$ (value of the dual norm).
\end{algorithmic}
\end{algorithm}
\begin{proposition}[Correctness and Complexity of Algorithm~\ref{algo:dual_norm}.]\label{prop:algphi}~\newline 
For~$\kappab$ in~$\Real^p$, the algorithm~\ref{algo:dual_norm} computes $\psip^*(\kappab)$ in at most $O(p|E'|)$ operations.
\end{proposition}
The proof is also presented in Appendix~\ref{appendix:proofs}. We note that
this worst-case complexity bound might be loose. We have indeed
observed in our experiments that the empirical complexity is close to be linear in~$|E'|$.
To concretely illustrate why computing the dual norm can be useful, we now
give optimality conditions for problem~(\ref{eq:prob}) involving~$\psip^*$. The following lemma can immediately
be derived from~\citet[][Proposition 1.2]{bach8}.
\begin{lemma}[Optimality Conditions for Problem~(\ref{eq:prob}) with~$\Omega=\psi$.]\label{lemma:opt}~\newline
 A vector~$\w$ be in~$\Real^p$ is optimal for problem~(\ref{eq:prob}) with~$\Omega=\psi$ if and only if
 \begin{displaymath}
    \psi^*(\nabla L(\w)) \leq \lambda~~~\text{and}~~~-\nabla L(\w)^\top \w = \lambda \psi(\w).
 \end{displaymath}
\end{lemma}
The next section presents an active-set type of algorithm~\citep[see][Chapter 6]{bach8}
building upon these optimality conditions and adapted to our penalty~$\psip$.

\subsection{Active Set Methods for Solving Problem~(\ref{eq:prob}) when~$\Omega=\psip$}\label{subsec:activeset}
As experimentally shown later in Section~\ref{sec:exp}, proximal gradient
methods allows us to efficiently solve medium-large/scale problems ($p <
100\,000$). Solving larger scale problems can, however, be more difficult.
Algorithm~\ref{algo:activeset} is an active-set strategy that can overcome
this issue when the solution is very sparse.
It consists of solving a sequence of smaller instances of
Equation~(\ref{eq:prob}) on subgraphs~$\tilde{G}=(\tilde{V},\tilde{E})$,
with~$\tilde{V} \subseteq V$ and~$\tilde{E} \subseteq E$, which we call
\emph{active graphs}.  It is based on the computation of the dual-norm~$\psip^*$,
which we have observed can empirically be obtained in a time linear or close to linear in~$|E'|$.
Given such a subgraph~$\tilde{G}$, we denote by~$\tilde{\GG}_p$ the set of
paths in~$\tilde{G}$. The subproblems the active set strategy involve are the following:
\begin{equation}
   \min_{\w \in \Real^p} \big\{ L(\w) + \lambda{\oldpsi}_{\tilde{\GG}_p}(\w)  \st \w_j = 0 ~~\text{for all}~j \notin \tilde{V}\big\}. \label{eq:subprob}
\end{equation}
The key observations are that (i) when~$\tilde{G}$ is small, subproblem~(\ref{eq:subprob}) is
easy to solve; (ii) after solving~(\ref{eq:subprob}), one can check optimality conditions
for problem~(\ref{eq:prob}) using Lemma~\ref{lemma:opt}, and update~$\tilde{G}$ accordingly.
Algorithm~\ref{algo:activeset} presents the full algorithm, and the next proposition ensures that it is correct.
\begin{algorithm}[!hbtp]
\caption{Active-set Algorithm for Solving Equation~(\ref{eq:prob}) with~$\Omega=\psip$.}\label{algo:activeset}
\begin{algorithmic}[1]
\STATE {\bf{Initialization}} $\w \leftarrow 0$; $\tilde{G} \leftarrow (\emptyset,\emptyset)$ (active graph); 
\LOOP
   \STATE Update~$\w$ by solving subproblem~(\ref{eq:subprob}) (using the current value of~$\w$ as a warm start);
   \STATE Compute $\tau \leftarrow \psip^*(\nabla L(\w))$ (using Algorithm~\ref{algo:dual_norm});
   \IF{$\tau \leq \lambda$} 
      \STATE {\bfseries exit} the loop;
   \ELSE
       \STATE $g \leftarrow \argmin_{g \in \GG_p} l_\tau(g)$ (shortest path problem in a directed acyclic graph);
       \STATE $\tilde{V} \leftarrow \tilde{V} \cup g$; $\tilde{E} \leftarrow \tilde{E} \cup \{ (u,v) \in E : u \in g ~~\text{and}~~v \in g\}$ (update of the active graph);
   \ENDIF
\ENDLOOP
\STATE {\bf{Return:}} $\w^\star \leftarrow \w$, solution to Equation~(\ref{eq:prob}).
\end{algorithmic}
\end{algorithm}
\begin{proposition}[Correctness of Algorithm~\ref{algo:activeset}.]\label{prop:activeset}~\newline 
Algorithm~\ref{algo:activeset} solves Equation~(\ref{eq:prob}) when~$\Omega=\psip$.
\end{proposition}
The proof is presented in Appendix~\ref{appendix:proofs}. It mainly relies on Lemma~\ref{lemma:opt},
which requires computing the quantity~$\psip^*(\nabla L(\w))$.
More precisely, it can be shown that when~$\w$ is a solution to subproblem~(\ref{eq:subprob}) for a subgraph~$\tilde{G}$,
whenever $\psip^*(\nabla L(\w)) \leq \lambda$, it is also a solution to the original large problem~(\ref{eq:prob}).
Note that variants of Algorithm~\ref{algo:activeset} can be considered: one can
select more than a single path~$g$ to update the subgraph~$\tilde{G}$, or one can
approximately solve the subproblems~(\ref{eq:subprob}). In the latter case, 
the stopping criterion could be relaxed in practice. One could use the criterion $\tau \leq \lambda+\varepsilon$, where
$\varepsilon$ is a small positive constant, or one could use a duality gap to stop
the optimization when the solution is guaranteed to be optimal enough.

In the next section, we present various experiments, illustrating how
the different penalties and algorithms behave in practice.

%% file: exp.tex
We now present experiments on synthetic, genomic and image data.  
Our algorithms have been implemented in C++ with a Matlab interface, they
have been made available as part of the open-source software package SPAMS, originally
accompanying~\citet{mairal7}.\footnote{The source code is available here: \url{http://spams-devel.gforge.inria.fr/}.}
We have implemented the proximal gradient algorithm FISTA 
\citep{beck} for convex regularization functions and ISTA for non-convex
ones.  When available, we use a relative duality gap as a stopping criterion
and stop the optimization when the relative duality gap is smaller
than~$10^{-4}$.  In our experiments, we often need to solve
Equation~(\ref{eq:prob}) for several values of the parameter~$\lambda$,
typically chosen on a logarithmic grid.  We proceed with a continuation
strategy: first we solve the problem for the largest value of~$\lambda$, whose
solution can be shown to be~$0$ when~$\lambda$ is large enough; then we decrease the
value of~$\lambda$, and use the previously obtained solution as
initialization.  This warm-restart strategy allows us to quickly
follow a regularization path of the problem.  For non-convex problems, it
provides us with a good initialization for a given~$\lambda$. The algorithm
ISTA with the non-convex penalty~$\varphip$ is indeed only guaranteed to
iteratively decrease the value of the objective function. As often the case
with non-convex problems, the quality of the optimization is subject to a good
initialization, and this strategy has proven to be important to obtain good
results.

As far as the choice of the parameters is concerned, we have observed that all
penalties we have considered in our experiments are very sensitive to the
regularization parameter~$\lambda$. Thus, we use in general a fine grid to choose
$\lambda$ using cross-validation or a validation set. Some of the penalties
involve an extra parameter, $\gamma$ in the case of $\varphip$ and~$\psip$.
This parameter offers some flexibility, for example it promotes the connectivity
of the solution for $\varphip$ and~$\psip$, but it also requires to be tuned
correctly to prevent overfitting. In practice, we have found the choice of the
second parameter always less critical than~$\lambda$ to obtain a good
prediction performance, and thus we use a coarse grid to choose this parameter.
All other implementation details are provided in each experimental section.

\subsection{Synthetic Experiments}
In this first experiment, we study our penalties~$\varphip$ and~$\psip$ 
in a controlled setting. Since generating synthetic graphs 
reflecting similar properties as real-life networks is difficult, we have
considered three ``real'' graphs of different sizes, which are part of the
10$^\text{th}$ DIMACS graph partitioning and graph clustering
challenge:\footnote{The datasets are available here: \url{http://www.cc.gatech.edu/dimacs10/archive/clustering.shtml}.}
\begin{myitemize}
   \item the graph~\textsf{jazz} was compiled by~\citet{gleiser} and represents a community network of jazz musicians. It contains $p=198$ vertices and~$m=2\,742$ edges;
   \item the graph~\textsf{email} was compiled by~\citet{guimera} and represents e-mail exchanges in a university. It contains $p=1\,133$ vertices and~$m=5\,451$ edges;
   \item the graph~\textsf{PGP} was compiled by~\citet{boguna} and represents information interchange among users in a computer network. It contains~$p=10\,680$ vertices and~$m=24\,316$ edges.
\end{myitemize}
We choose an arbitrary topological ordering for all of these graphs, orient the
arcs according to this ordering, and obtain DAGs.\footnote{A topological
ordering~$\preceq$~of vertices in a directed graph is such that if there is an
arc from vertex~$u$ to vertex $v$, then~$u \prec v$. A directed graph is
acyclic is and only if it possesses a topological
ordering~\citep[see][]{ahuja}.} We generate structured sparse linear models
with measurements corrupted by noise according to different scenarii, and compare the ability of different
regularization functions to recover the noiseless model. More precisely, we
consider a design matrix~$\X$ in~$\Real^{n \times p}$ with less observations than predictors ($n \defin \lfloor
p/2 \rfloor$), and whose entries are i.i.d.  samples from a normal
distribution~${\mathcal N}(0,1)$.  For simplicity, we preprocess each column
of~$\X$ by removing its mean component and normalize it to have unit
$\ell_2$-norm.  Then, we generate sparse vectors~$\w_0$ with~$k$ non-zero
entries, according to different strategies described in the sequel.
We synthesize an observation vector~$\y=\X\w_0 + \varepsilonb$ in~$\Real^n$,
where the entries of~$\varepsilonb$ are i.i.d. draws from a normal
distribution~${\mathcal N(0,\sqrt{k / n} \sigma)}$, with different noise
levels:
\begin{myitemize}
   \item \textsf{high SNR}: we choose $\sigma=0.2$; this yields a signal noise ratio (SNR) $\|\X\w_0\|_2^2 / \|\varepsilonb\|_2^2$
   of about~$26$. We note that for~$\sigma \leq 0.1$ almost all penalties
   almost perfectly recover the true pattern;
   \item \textsf{medium SNR}: for $\sigma=0.4$, the SNR is about $6$;
   \item \textsf{low SNR}: for $\sigma=0.8$, the SNR is about $1.6$.
\end{myitemize}

Choosing a good criterion for comparing different penalties is difficult, and a
conclusion drawn from an experiment is usually only valid for a given
criterion.  For example, we present later an image denoising experiment in
 Section~\ref{subsec:image}, where non-convex penalties outperform convex ones
 according to one performance measure, while being the other way around for
 another one. 
In this experiment, we choose the relative mean square error
$\|\X\hatw-\X\w_0\|_2^2$ as a criterion, and use ordinary least square (OLS) to
refit the models learned using the penalties.  Whereas OLS does not change the
results obtained with the non-convex penalties we consider, it changes
significantly the ones obtained with the convex ones.  In practice, OLS
counterbalances the ``shrinkage'' effect of convex penalties, and
empirically improves the results quality for low noise regimes (high SNR), but
deteriorates it for high noise regimes ({low SNR}). 

For simplicity, we also assume (in this experiment only) that an oracle gives us the
optimal regularization parameter~$\lambda$, and therefore the conclusions we
draw from the experiment are only the existence or not of good solutions on the
regularization path for every penalty.  A more exhaustive comparison would
require testing different combinations (with OLS, without OLS) and
different criteria, and using internal cross-validation to select the
regularization parameters. This would require a much heavier computational
setting, which we have chosen not to implement in this experiment.  After
obtaining the matrix~$\X$, we propose several strategies to generate ``true''
models~$\w_0$:
\begin{myitemize}
   \item in the scenario~\textsf{flat} we randomly select~$k$ entries
   without exploiting the graph structure;
   \item the scenario~\textsf{graph}
   consists of randomly selecting~$5$ entries, and iteratively selecting new
   vertices that are connected in~$G$ to at least one previously selected
   vertex. This produces fairly connected sparsity patterns, but does not
   exploit arc directions;
   \item the scenario~\textsf{path} is similar to~\textsf{graph}, but we
   iteratively add new vertices following single paths in~$G$. It exploits 
   arc directions and produces sparsity patterns that can be covered by a small
   number of paths, which is the sort of patterns that our path-coding penalties
   encourage.
\end{myitemize}
The number of non-zero entries in~$\w_0$ is chosen to be $k \defin \lfloor 0.1p
\rfloor$ for the different graphs, resulting in a fairly sparse vector.
The values of the non-zero entries are randomly chosen in~$\{-1,+1\}$. 
We consider the formulation~(\ref{eq:prob}) where~$L$ is the square loss: $L(\w)=\frac{1}{2}\|\y-\X\w\|_2^2$ and~$\Omega$ is one of the following penalties:
\begin{myitemize}
   \item the classical~$\ell_0$- and~$\ell_1$-penalties;
   \item the penalty~$\psi$ of~\citet{jacob} where the groups~$\G$ are pairs of vertices linked by an~arc;
   \item our path-coding penalties~$\varphip$ or~$\psip$ with the weights~$\eta_g$ defined in~(\ref{eq:choice}).
   \item the penalty of~\citet{huang}, and their strategy to encourage sparsity pattern
   with a small number of connected components. We use their implementation of the greedy
   algorithm
   StructOMP\footnote{The source code is available here: \url{http://ranger.uta.edu/~huang/R_StructuredSparsity.htm}}.
   This algorithm uses a strategy dubbed ``block-coding'' to approximately deal
   with this penalty~\citep[see][]{huang}, and has an additional parameter, which we also denote by~$\mu$, to control
   the trade-off between sparsity and connectivity.
\end{myitemize}
For every method except StructOMP, the
regularization parameter~$\lambda$ is chosen among the values~$2^{i/4}$,
where~$i$ is an integer. We always start from a large value
for~$i$, and decrease its value by one, following the regularization
path. For the penalties~$\varphip$ and~$\psip$, the parameter~$\gamma$ is
simply chosen in~$\{1/4,1/2,1,2,4\}$.
Since the algorithm StructOMP is greedy and iteratively increases the model
complexity, we record every solution obtained on the regularization path during
one pass of the algorithm. Based on information-theoretic arguments,~\citet{huang}
propose a default value for their parameter~$\mu=1$. Since changing this parameter value empirically 
improves the results quality, we try the values $\{1/4,1/2,1,2,4\}$ for
a fair comparison with our penalties~$\varphip$ and~$\psip$.

We report the results for the three graphs, three scenarii for
generating~$\w_0$, three noise levels and the five penalties in
Figure~\ref{figure:synthetic}. We report on these graphs the ratio between the
prediction mean square error and the best achievable error if the sparsity
pattern was given by an oracle. In other words, denoting
by~$\hatw^{\text{oracle}}$ the OLS estimate if an oracle gives us the sparsity
pattern, we report the value $\|\X\hatw-\X\w_0\|_2^2 /
\|\X\hatw^{\text{oracle}}-\X\w_0\|_2^2$. The best achievable value for this
criterion is therefore~$1$, which is represented by a dotted line on all
graphs. We reproduce the experiment~$20$ times, randomizing every step,
including the way the vector~$\w_0$ is generated to obtain error bars representing
one standard deviation.

We make pairwise comparisons and 
statistically assess our conclusions using error bars
or, when needed, paired one-sided T-tests with a~$1\%$ significance level. The comparisons are the following:
\begin{itemize}
   \item {\bfseries convex vs non-convex ($\ell_0$ vs $\ell_1$ and~$\varphip$ vs $\psip$)}: \emph{For high SNR, non-convex penalties do significantly better than convex ones, whereas it is the other way around for low SNR.} The differences are highly significant for the graphs~\textsf{email} and~\textsf{PGP}. For medium SNR, conclusions are mixed, either the difference between a convex penalty and its non-convex counterpart are not significant or one is better than another. 
   \item {\bfseries unstructured vs path-coding ($\ell_0$ vs $\varphip$ and $\ell_1$ vs $\psip$)}: 
   \emph{In the structured scenarii \textsf{graph} and~\textsf{path}, the structured penalties $\varphip$ and~$\psip$ respectively do better than~$\ell_0$ and~$\ell_1$.
   In the unstructured~\textsf{flat} scenario,~$\ell_0$ and~$\ell_1$ should be preferred.}
   More precisely, for  the scenarii \textsf{graph} and~\textsf{path},
   $\varphip$ and~$\psip$ respectively outperform $\ell_0$ and~$\ell_1$ with
   statistically significant differences, except (i) for high SNR,
   both~$\varphip$ and~$\ell_0$ achieve perfect recovery; (ii) with the
   smallest graph \textsf{jazz}, the $p$-values obtained to compare
   $\psip$ vs $\ell_1$ are slightly above our~$1\%$ significance level.  In
   the~\textsf{flat} scenario, $\ell_0$ and~$\varphip$ give similar
   results, whereas~$\psip$ performs slightly worse than~$\ell_1$ in general
   even though they are very close.
   \item {\bfseries \citet{jacob} vs path-coding ($\psi$ with pairs vs $\psip$)}: \emph{in the scenario~\textsf{path}, $\psip$ outperforms $\psi$ (pairs). It is generally also the case in the scenario~\textsf{graph}}. The differences are always significant in the low SNR regime and with the largest graph~\textsf{PGP}.
   \item {\bfseries \citet{huang} vs path-coding (StructOMP vs $\varphip,\psip$)}: \emph{For the scenario~\textsf{path}, either $\varphip$ (for high SNR) or~$\psip$ (for low SNR) outperform StructOMP}. 
   \emph{For the scenario~\textsf{graph}, the best results are shared between StructOMP and our penalties for high and medium SNR, and our penalties do better for low SNR}.
   More precisely in the scenario~\textsf{graph}: (i) there is no significant difference for high SNR
   between~$\varphip$ and StructOMP; (ii) for medium SNR, StructOMP does slightly
   better with the graph~\textsf{PGP} and similarly as~$\varphi$ for the two
   other graphs; (iii) for low SNR, our penalties do better than StructOMP with the
   two largest graphs~\textsf{email} and~\textsf{PGP} and similarly with the
   smallest graph~\textsf{jazz}.
\end{itemize}
To conclude this experiment, we have shown that our penalties offer a
competitive alternative to StructOMP and the penalty of~\citet{jacob},
especially when the ``true'' sparsity pattern is exactly a union of a few paths
in the graph. We have also identified different noise and size regimes, where a
penalty should be preferred to another.  Our experiment also
shows that having both a non-convex and convex variant of a penalty
can be interesting. In low SNR, convex penalties are indeed better
behaved than non-convex ones, whereas it is the other way around when the SNR
is high.
\begin{figure}[h]
\centering
\includegraphics[width=0.325\linewidth]{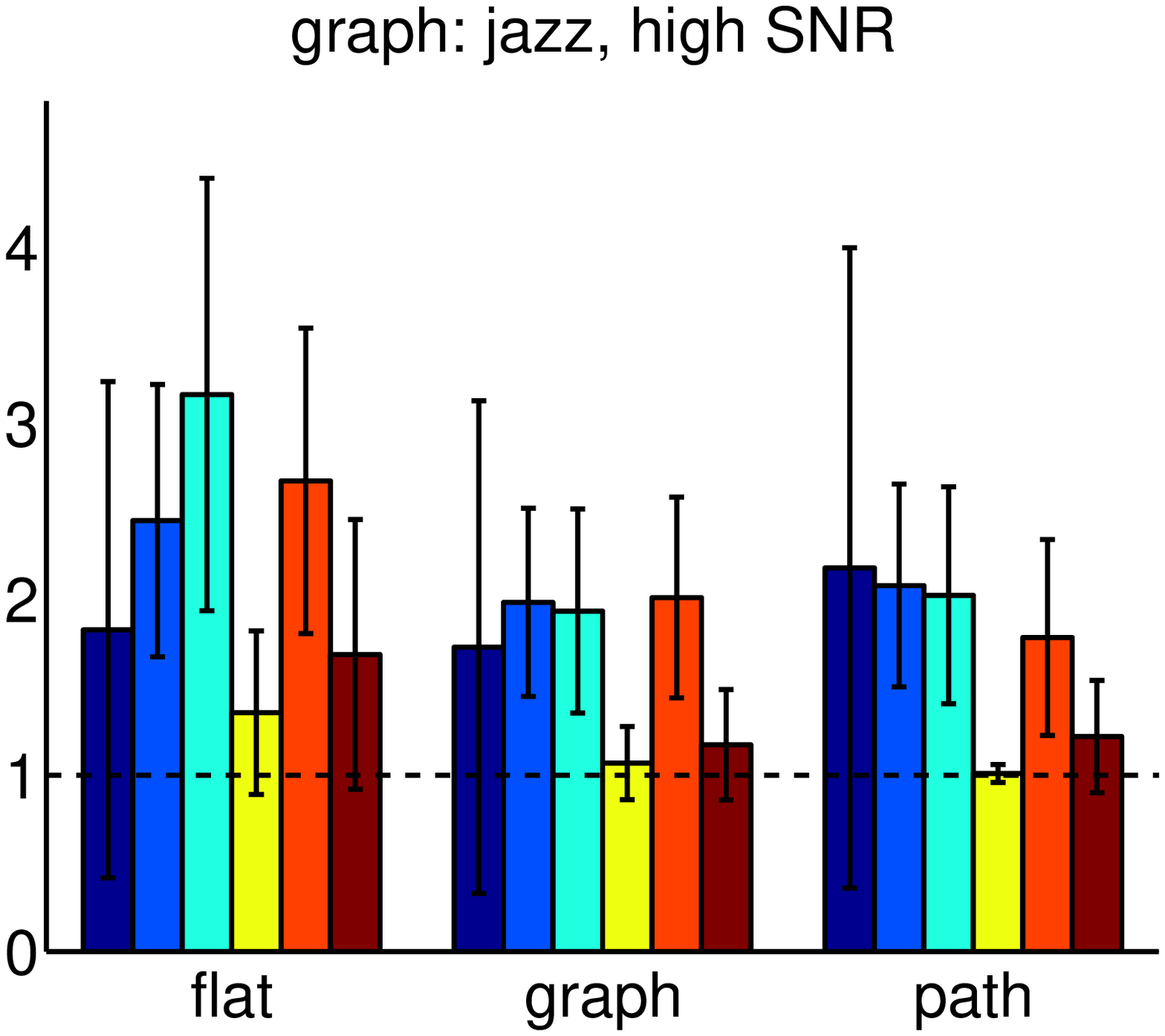} \hfill
\includegraphics[width=0.325\linewidth]{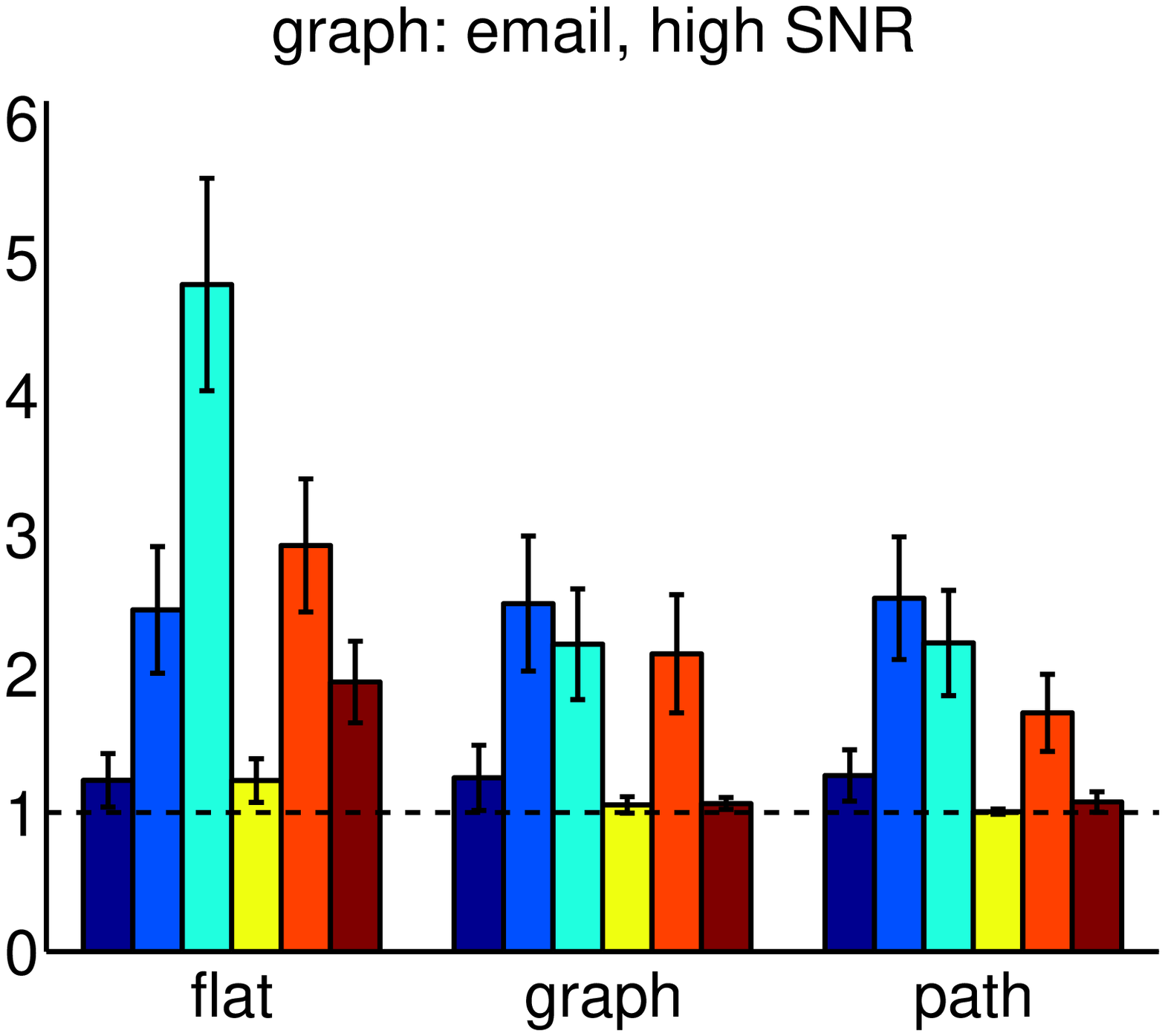} \hfill
\includegraphics[width=0.325\linewidth]{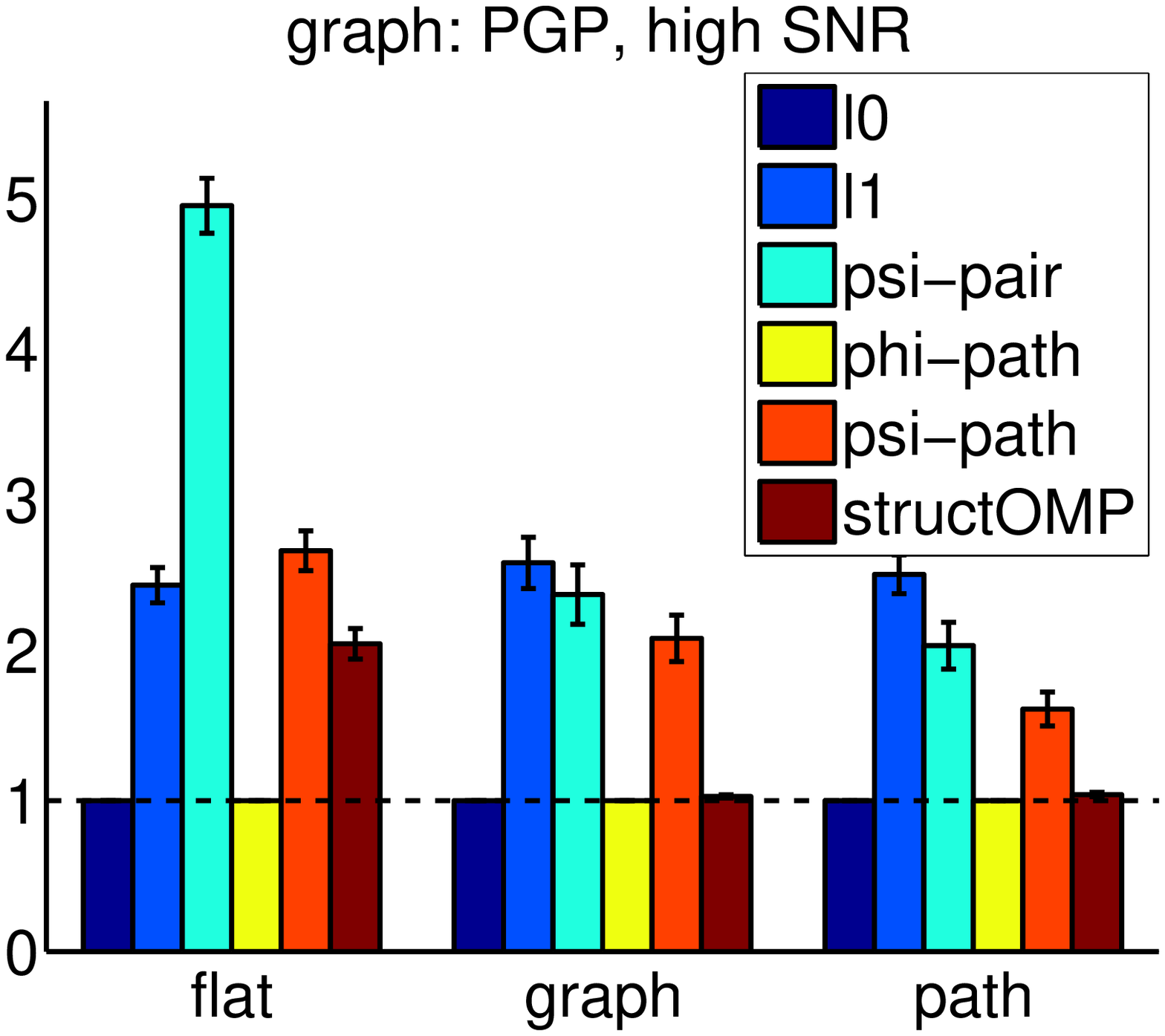} \\
\includegraphics[width=0.325\linewidth]{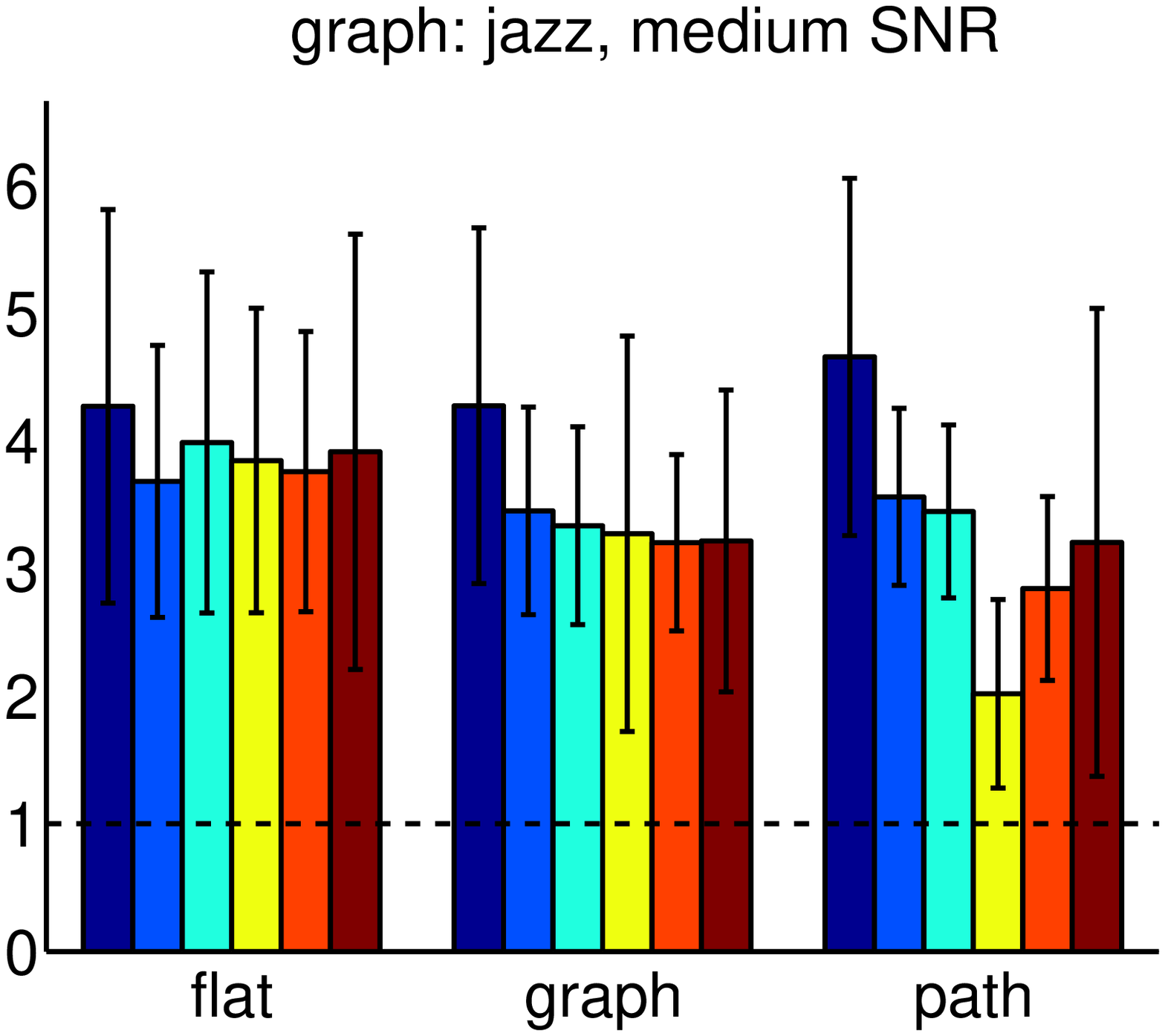} \hfill
\includegraphics[width=0.325\linewidth]{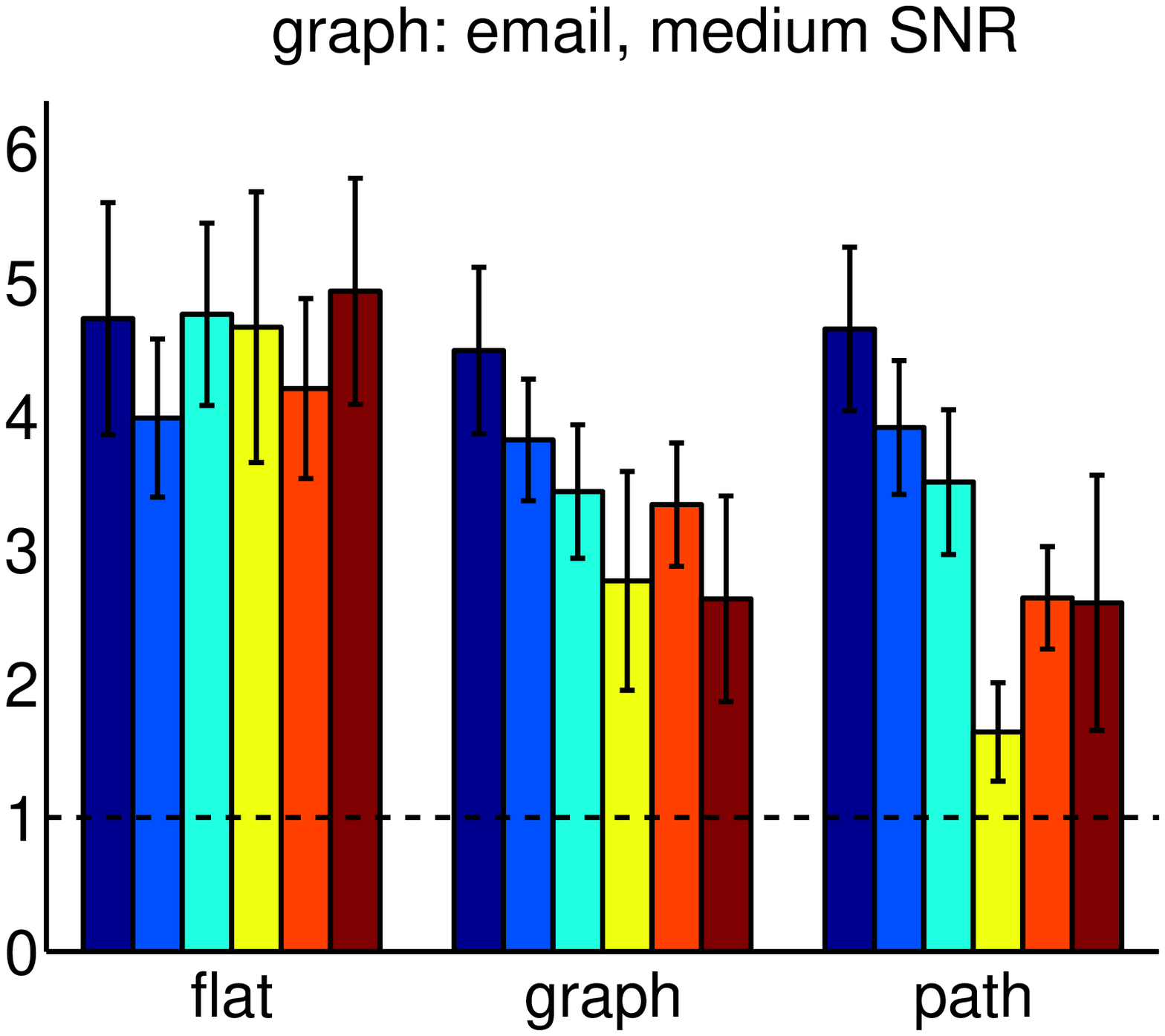} \hfill
\includegraphics[width=0.32\linewidth]{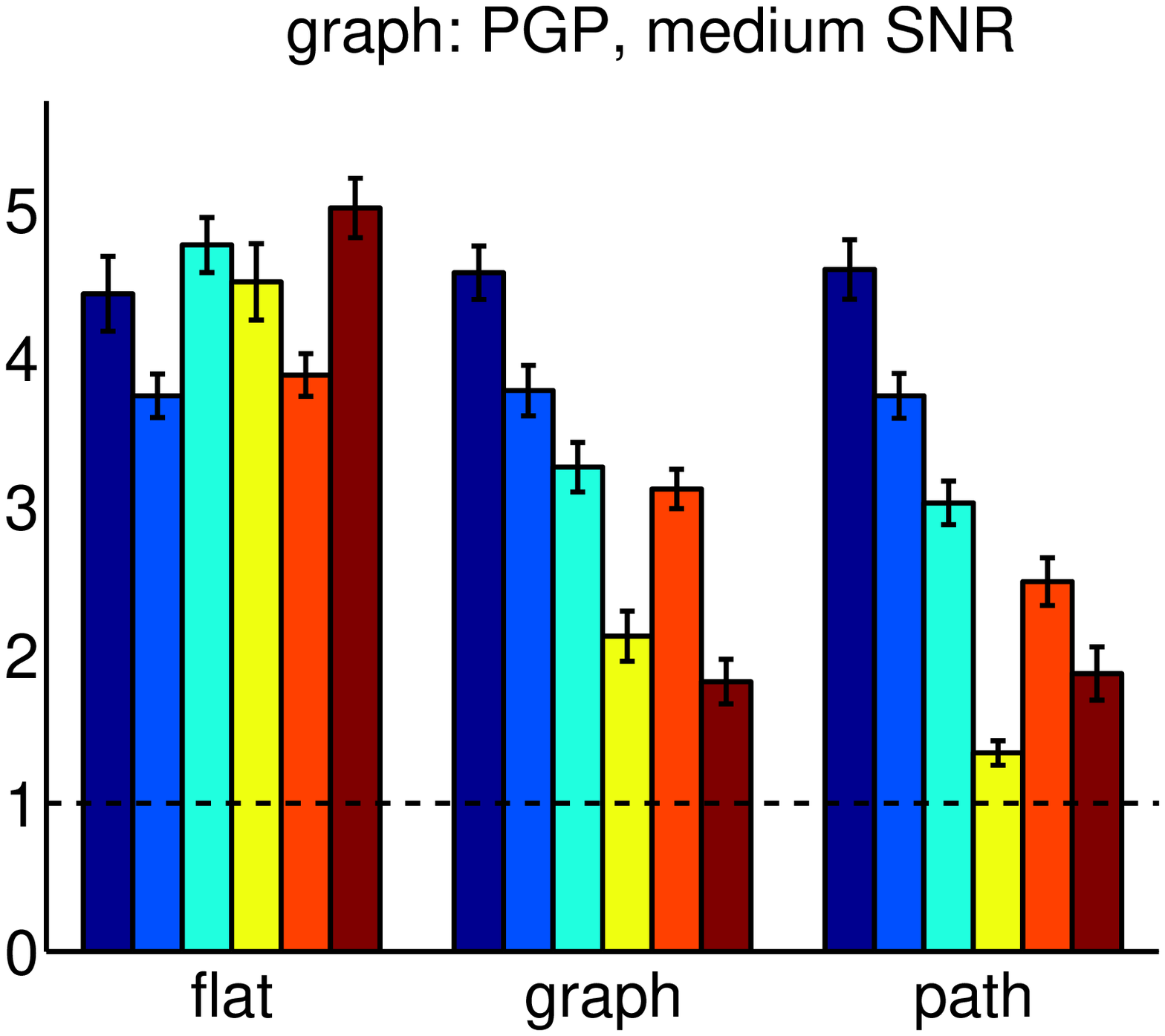} \\
\includegraphics[width=0.325\linewidth]{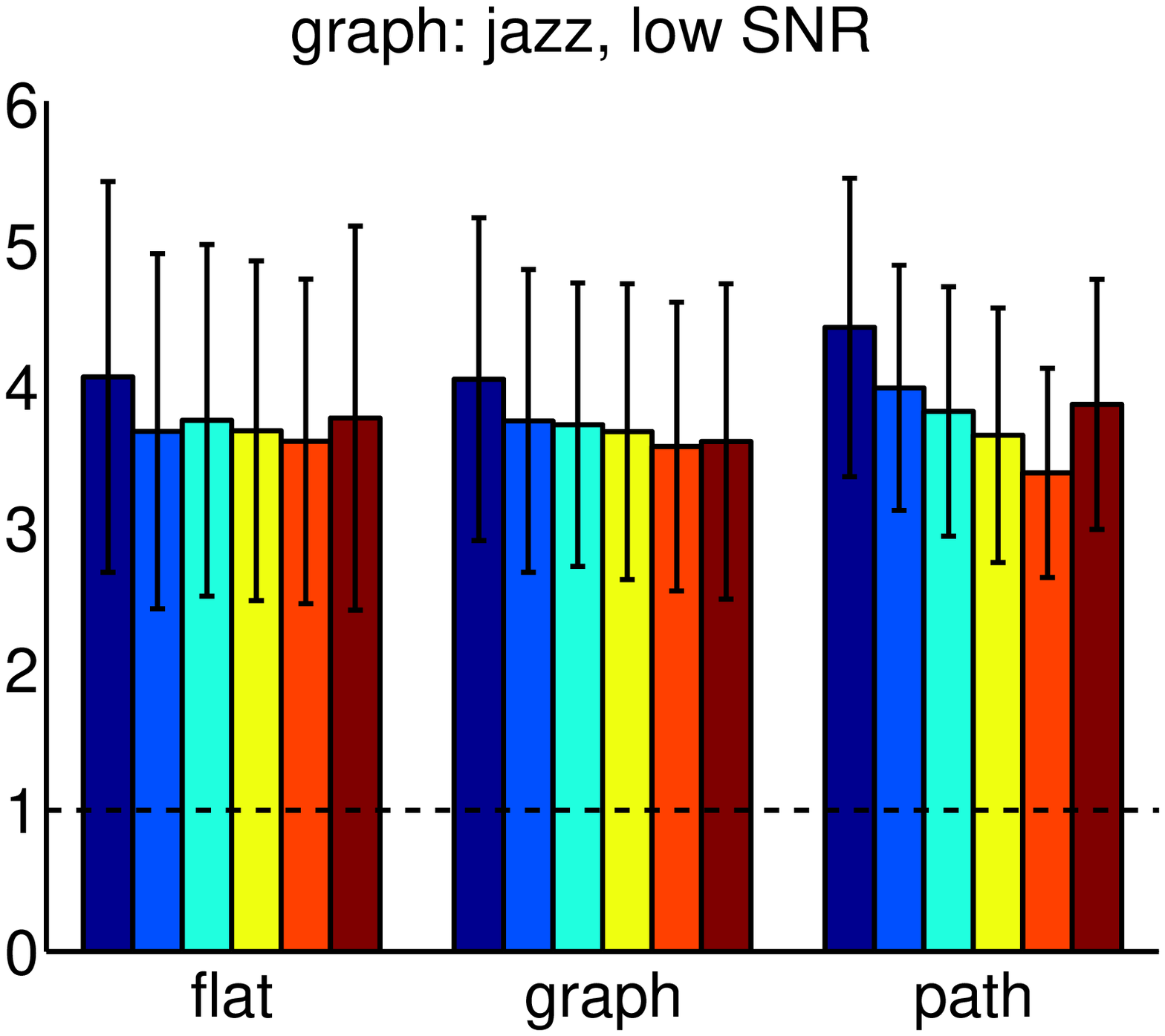} \hfill
\includegraphics[width=0.325\linewidth]{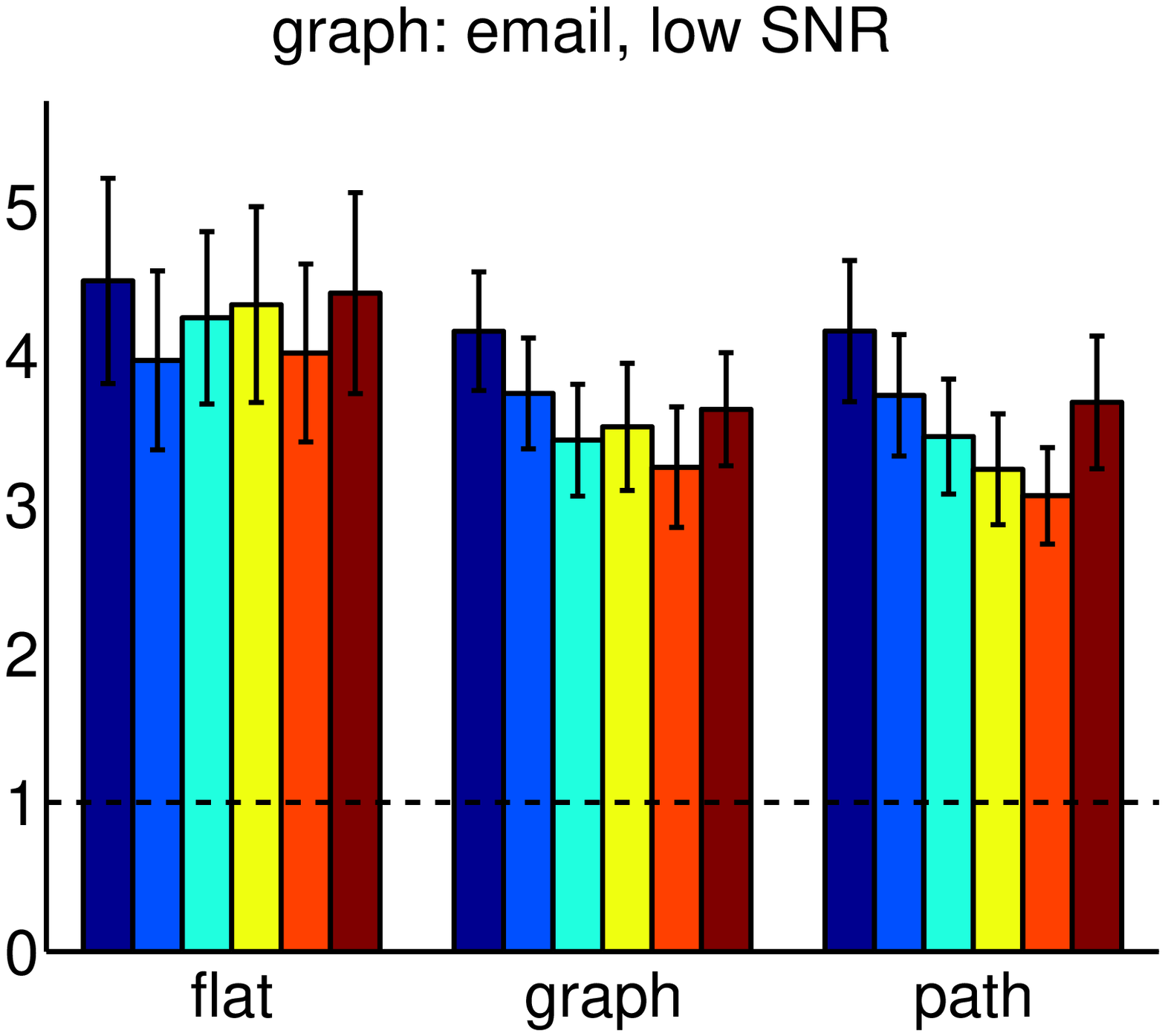} \hfill
\includegraphics[width=0.325\linewidth]{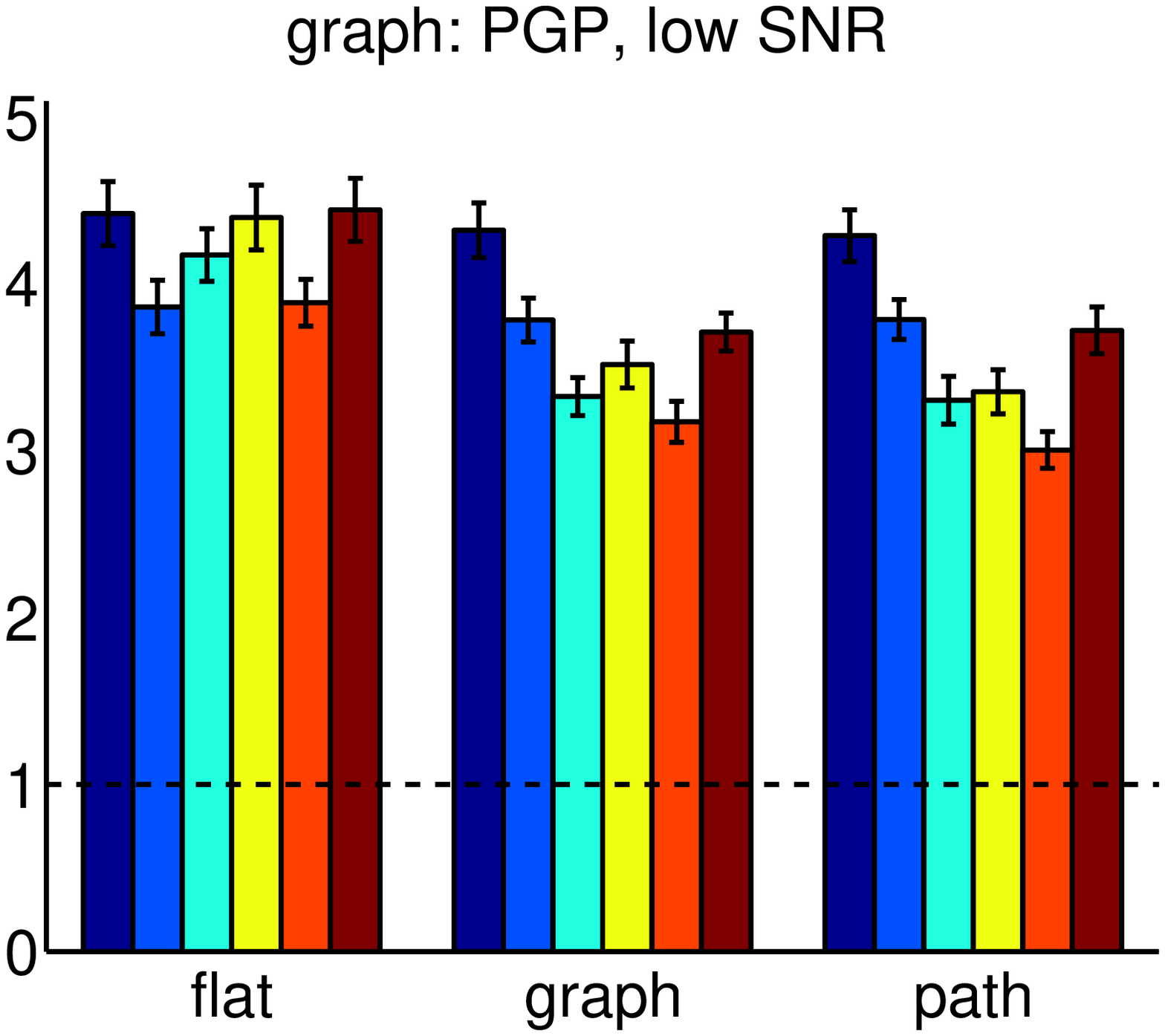} 
\caption{Every bar represents the ratio between the mean square error estimate and the oracle mean square error estimate (see main text for an explicit formula and the full experimental setting). The error bars represent one standard deviation. Each row corresponds to a specific noise level, and every column to a different graph. For a specific noise level and specific graph, the results for three scenarii, \textsf{flat}, \textsf{graph} and~\textsf{path} are reported. Each group of six bars represents the results obtained with six penalties, from left to right: $\ell_0$, $\ell_1$, $\psi$ (with~$\G$ being the pairs of vertices linked by an arc), $\varphip$, $\psip$, and the method StructOMP. A legend is presented in the top right figure. } \label{figure:synthetic}
\end{figure}

\subsection{Image Denoising}\label{subsec:image}
State-of-the-art image restoration techniques are often exploiting a good model
for small image patches~\citep{elad,dabov2,mairal8}.
We consider here the task of denoising natural images corrupted by white
Gaussian noise, following an approach introduced by~\citet{elad}. It consists of
the following steps: 
\begin{enumerate}
\item extract all overlapping patches~$(\y^i)_{i=1,\ldots,m}$ from a
noisy image;
\item compute a sparse approximation of every individual patch $\y^i$:\label{step2}
\begin{equation}
   \min_{\w^i \in \Real^p} \Big[\frac{1}{2}\|\y^i - \X\w^i\|_2^2 + \lambda \Omega(\w^i)\Big],\label{eq:denoisepatch}
\end{equation}
where the matrix~$\X=[\x^1,\ldots,\x^p]$ in~$\Real^{n \times p}$ is a
predefined ``dictionary'', $\lambda \Omega$ is a sparsity-inducing regularization and the
term~$\X\w^i$ is the clean estimate of the noisy patch~$\y^i$;
\item since the patches overlap, each pixel admits several estimates. The last
step consists of averaging the estimates of each pixel to reconstruct the full
image. \label{step3}
\end{enumerate}

Whereas~\citet{elad} learn an overcomplete basis set to obtain a ``good''
matrix~$\X$ in the step \ref{step2} above, we choose a simpler approach
and use a classical orthogonal discrete cosine transform (DCT) dictionary~$\X$~\citep{ahmed}.
We present such a dictionary in
Figure~\ref{fig:dct} for~$8 \times 8$ image patches.
As shown in the figure, 
DCT elements can be organized on a two-dimensional grid and ordered by horizontal
and vertical frequencies. We use the DAG structure 
connecting neighbors on the grid, going from low to high frequencies.
In this experiment, we
address the following questions:
\begin{itemize}
\item[{\bfseries (A)}] \emph{In terms of optimization, is our approach
efficient for this experiment?} Because the number of problems to solve is large (several millions),
the task is difficult.
\item[{\bfseries (B)}] \emph{Do we get better results by using the graph
   structure than with classical~$\ell_0$- and~$\ell_1$-penalties?}  
\item[{\bfseries (C)}] \emph{How does the method compare with the state of the art?}  
\end{itemize}
\begin{figure}[btp]
\centering
\begin{tikzpicture}[scale=0.7]
\begin{scope}
    \foreach \i in { 0, 1, 2, 3, 4, 5, 6 } {
       \foreach \j in { 1, 2, 3, 4, 5, 6, 7 } {
          \draw[arrow] (\i,\j) -- (\i,\j-0.7);
          \draw[arrow] (\i,\j) -- (\i+0.7,\j);
       }
    }
    \foreach \j in { 1, 2, 3, 4, 5, 6, 7 } {
       \draw[arrow] (7,\j) -- (7,\j-0.7);
    }
    \foreach \i in { 0, 1, 2, 3, 4, 5, 6 } {
       \draw[arrow] (\i,0) -- (\i+0.7,0);
    }
    \newcounter{test}; 
    \setcounter{test}{0};
    \foreach \j  in { 7, 6, 5, 4, 3, 2, 1, 0 } {
       \foreach \i in { 0, 1, 2, 3, 4, 5, 6, 7 } {
          \addtocounter{test}{1};
          \node [thick]  (toto) at (\i,\j) {\includegraphics[width=0.5cm]{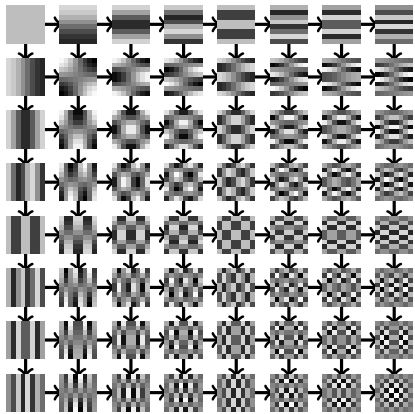}};
       }
    }
\end{scope}
\end{tikzpicture}
\caption{Orthogonal DCT dictionary with $n=8 \times 8$ image patches. The dictionary elements are organized by horizontal and vertical frequencies.} \label{fig:dct}
\end{figure}
Note that since the dictionary~$\X$ in~$\Real^{n \times p}$ is orthogonal, the non-convex problems we
address here are solved exactly. Indeed, 
it can be
shown that Equation~(\ref{eq:denoisepatch}) is equivalent~to
\begin{displaymath}
   \min_{\w^i \in \Real^p} \Big[\frac{1}{2}\|\X^\top\y^i - \w^i\|_2^2 + \lambda \Omega(\w^i)\Big],
\end{displaymath}
and therefore the solution admits a closed form~$\w^{i\star}\!\defin\!\text{Prox}_{\lambda\Omega}(\X^\top \y^i)$. For $\ell_0$ and $\ell_1$, the
 solution is respectively obtained by hard and soft-thresholding,
and we have introduced some tools in Section~\ref{sec:approach} to compute
the proximal operators of~$\varphip$ and~$\psip$.
We consider~$e \times e$ image patches, with~$e \in
\{6,8,10,12,14,16\}$, and a parameter~$\lambda$ on a logarithmic scale with
step~$2^{1/8}$.  We also exploit the variant of our penalties presented in
Section~\ref{sec:approach} that allows us to choose the costs on the arcs of the 
graph~$G'$. We choose here a small cost on the arc~$(s,1)$ of
the graph~$G'$, and a large one for every arc~$(s,j)$, for~$j$
in~$\{2,\ldots,p\}$, such that all paths selected by our approach are
encouraged to start by the variable~$1$ (equivalently the dictionary
element~$\x^1$ with the lowest frequencies). 
We use a dataset of~$12$ classical high-quality images (uncompressed and free
of artifact). We optimize the parameters~$\lambda$ and $e$ on the first~$3$
images, keeping the~$9$ last images as a test set and report denoising results
on Table~\ref{table:denoise}. Even though this dataset is relatively small,
it is standard in the image processing literature, making the
comparison easy with other approaches.\footnote{This dataset can be 
found for example in~\citet{mairal8}.}

We start by answering question~{\bfseries (A)}: we have observed that the time of
computation depends on several factors, including the problem size
and the sparsity of the solution (the sparser, the faster).
In the setting~$\sigma\!=\!10$ and~$e\!=\!8$, we are able to 
denoise approximately~$4\,000$ patches per second using~$\varphip$,
and~$1\,800$ in the setting~$\sigma\!=\!50$ and~$e\!=\!14$ with our
laptop 1.2Ghz CPU (core i3 330UM). The penalty~$\psip$ requires solving
quadratic minimum cost flow problems, and is slower to use in practice:
the numbers $4\,000$ and~$1\,800$ above respectively become~$70$ and~$130$. Our
approach with~$\varphip$ proves therefore to be fairly efficient for our
task, allowing us to process an image with about $250\,000$ patches in between
one and three minutes. As expected, simple penalties are faster to use: about $65\,000$ patches
per second can be processed using $\ell_0$.

Moving to question~{\bfseries (B)}, the best performance among the penalties
$\ell_0$, $\ell_1$, $\varphip$ and~$\psip$ is obtained by~$\varphip$.
This difference is statistically significant: we measure for instance an
average improvement of $0.38\pm0.21$ dB over~$\ell_0$ for~$\sigma \geq 20$.
For this denoising task, it is indeed typical to have non-convex penalties
outperforming convex ones~\citep[see][Section 1.6.5, for a benchmark
between~$\ell_0$ and~$\ell_1$-penalties]{mairal_thesis}, and this is why the
original method of~\citet{elad} uses the~$\ell_0$-penalty.  Interestingly, this
superiority of non-convex penalties in this denoising scheme based on
overlapping patches is usually only observed after the averaging
step~\ref{step3}. When one evaluates the quality of the denoising of individual
patches without averaging---that is, after step~\ref{step2}, opposite
conclusions are usually drawn~\citep[see again][Section 1.6.5]{mairal_thesis}. We
therefore report mean-square error results for individual patches without
averaging in Table~\ref{table:denoisepatches} when~$e=10$. As expected, the
penalty~$\psip$ turns out to be the best at this stage of the algorithm.
Note that the bad results obtained by convex penalties 
after the averaging step are possibly due to the shrinkage effect of these
penalties. It seems that the shrinkage is helpful for denoising individual patches, but hurts after the averaging process.

We also present the performance of state-of-the-art image denoising approaches
in Table~\ref{table:denoise} to address question~{\bfseries (C)}. We have
chosen to include in the comparison several methods that have successively been
considered as the state of the art in the past: the Gaussian Scale Mixture (GSM) method
of~\citet{portilla}, the K-SVD algorithm of~\citet{elad}, the BM3D method
of~\citet{dabov2} and the sparse coding approach of~\citet{mairal8} (LSSC).
We of course do not claim to do better than the most recent approaches
of~\citet{dabov2} or \citet{mairal8}, which in addition to sparsity exploit non-local self
similarities in images~\citep{buades}. Nevertheless, given  the fact that we use a simple orthogonal DCT
dictionary, unlike~\citet{elad} who learn overcomplete dictionaries adapted to the image, our approach based
on the penalty~$\varphip$ performs relatively well. We indeed obtain similar
results as~\citet{elad} and \citet{portilla}, and show that structured parsimony could be
a promising tool in image processing.
\begin{table}[hbtp]
\centering
\begin{tabular}{|c|c|c|c|c|c|c|c|}
\hline
 $\sigma$  & 5 & 10 & 15 & 20 & 25 & 50 & 100 \\
\hline
\hline
\multicolumn{8}{|c|}{Our approach} \\
\hline
$\ell_0$   & \textbf{37.04} & 33.15          & 31.03          & 29.59          & 28.48          & 25.26          & 22.44  \\
\hline
$\ell_1$   & 36.42          & 32.28          & 30.06          & 28.59          & 27.51          & 24.48          & 21.96  \\
\hline
$\varphip$  & 37.01          & \textbf{33.22} & \textbf{31.21} & \textbf{29.82} & \textbf{28.77} & \textbf{25.73} & \textbf{22.97} \\
\hline
$\psip$     & 36.32          & 32.17          & 29.99          & 28.54          & 27.49          & 24.54          & 22.12  \\
\hline
\hline
\multicolumn{8}{|c|}{State-of-the-art approaches} \\
\hline
\citealp{portilla} (GSM) &   36.96 &  33.19  & 31.17 & 29.78 & 28.74  & 25.67 & 22.96  \\
\hline
\citealp{elad} (K-SVD) &   37.11  & 33.28   & 31.22 &  29.81  &  28.72  &  25.29 &  22.02   \\
\hline
\citealp{dabov2} (BM3D) &  37.24  &   33.60 & 31.68  &  \textbf{30.36}  &  \textbf{29.36}  &  26.11 &  23.11   \\
\hline
\citealp{mairal8} (LSSC) & \textbf{37.29}   & \textbf{33.64}  &  \textbf{31.70} &  \textbf{30.36}  &   29.33    &  \textbf{26.20}     & \textbf{23.20}  \\
\hline
\end{tabular}
\caption[Denoising results]{Denoising results for~$9$ test images. The numbers represent the
   average PSNR in dB (higher is better). Denoting by \textrm{MSE} the
    mean-squared-error for images whose intensities are between $0$ and
    $255$, the \textrm{PSNR} is defined as $\textrm{PSNR}=10\log_{10}( 255^2 /
 \textrm{MSE} )$. 
      Pixel values are scaled between $0$ and~$255$ and $\sigma$ (the standard deviation of the noise) is between~$5$ and~$100$. The top part of the table presents the results of the denoising scheme obtained with different penalties. The bottom part presents the results obtained with various state-of-the-art denoising methods (see main text for more details). Best results are in bold for both parts of the table.} \label{table:denoise}
\end{table}
\begin{table}[t]
\centering
\begin{tabular}{|c|c|c|c|c|c|c|c|}
\hline
 $\sigma$  & 5 & 10 & 15 & 20 & 25 & 50 & 100 \\
\hline
\hline
$\ell_0$ & 3.60  & 10.00 & 16.65  & 23.22 & 29.58 & 57.97 & 95.79 \\
\hline
$\ell_1$ & 2.68  & 7.65 & 13.42   & 19.22 & 25.23 & 52.38 & 87.90 \\
\hline
$\varphip$ & 3.26  & 8.36 & 13.62 & 18.83 & 23.99 & 47.66 & 84.74 \\
\hline
$\psip$ & \textbf{2.66} & \textbf{7.27} & \textbf{12.29} & \textbf{17.35} & \textbf{22.65} & \textbf{45.04} & {\bfseries 76.85} \\
\hline
\end{tabular}
\caption{Denoising results for individual~$10 \times 10$ image patches on the~$9$ test images. The numbers represent the mean-squared error for the image patches (lower the better). Best results are in bold.} \label{table:denoisepatches}
\end{table}

\subsection{Breast Cancer Data}\label{subsec:cancer}
One of our goal was to develop algorithmic tools improving the
approach of~\citet{jacob}. It is therefore natural to try one of the dataset
they used to obtain an empirical comparison.  On the one hand, we have
developed tools to enrich the group structure that the penalty~$\psi$
could handle, and thus we expect better results. On the other hand, the 
graph in this experiment is undirected and we need to use heuristics to transform it
into a DAG.

We use in this task the breast cancer dataset of~\citet{vijver}. It
consists of gene expression data from $8\,141$ genes in $n\!=\!295$ breast
cancer tumors and the goal is to classify metastatic samples versus
non-metastatic ones.  Following~\citet{jacob}, we use the gene network compiled
by~\citet{chuang}, obtained by concatenating several known biological networks. 
As argued by~\citet{jacob}, taking into account the graph structure into the
regularization has two objectives: (i) possibly improving the
prediction performance by using a better prior; (ii) identifying connected
subgraphs of genes that might be involved in the metastatic form of the
disease, leading to results that yield better interpretation than the selection of isolated genes.
Even though prediction is our ultimate goal in this task, interpretation is
equally important since it is necessary in practice to design drug targets.
In their paper,~\citet{jacob} have succeeded in the sense that their penalty
is able to extract more connected patterns than the~$\ell_1$-regularization, even though
they could not statistically assess significant improvements in terms of
prediction.  Following~\citet{jacob}, we also assume that connectivity of the
solution is an asset for interpretation.  The questions we address here are
the following:
\begin{itemize}
\item[\textbf{(A)}] \emph{Despite the heuristics described below to transform the graph into a DAG, does our approach lead to well-connected solutions in the original (undirected) graph? Do our penalties lead to better-connected solutions than other penalties?}.
\item[\textbf{(B)}] \emph{Do our penalties lead to better classification performance than~\citet{jacob} and other classical unstructured and structured regularization functions? Is the graph structure useful to improve the prediction? Does sparsity help prediction?}
\item[\textbf{(C)}] \emph{How efficient is our approach for this task?} The problem here is of medium/large scale but should be solved a large number of times (several thousands of times) because of the internal cross-validation procedure.
\end{itemize}
The graph of genes, which we denote by~$G_0$, contains~$42\,587$ edges, and,
like~\citet{jacob}, we keep the~$p\!=\!7\,910$~genes that are present in~$G_0$.
In order to obtain interpretable results and select connected components
of~$G_0$, \citet{jacob} use their structured sparsity penalty~$\psi$ where the
groups~$\G$ are all pairs of genes linked by an arc. 
Our approach requires a DAG, but we will show in the sequel that we
nevertheless obtain good results after heuristically transforming~$G_0$ into a
DAG.  To do so, we first treat~$G_0$ as directed by choosing random
directions on the arcs, and second we remove some arcs along cycles in the graph.
It results in a DAG containing~$33\,303$ arcs, which we denote by~$G$.  This
pre-processing step is of course questionable since our penalties are originally not
designed to deal with the graph~$G_0$.  We of course do not claim to be able to
individually interpret each path selected by our method, but, as we show, it
does not prevent our penalties~$\varphip$ and~$\psip$ to achieve their ultimate
goal---that is connectivity in the original graph~$G_0$.

We consider the formulation~(\ref{eq:prob}) where~$L$ is a weighted 
logistic regression loss:
\begin{displaymath}
   L(\w) \defin \sum_{i=1}^n \frac{1}{n_{y_i}}\log(1+ e^{-y_i \w^\top \x_i}),
\end{displaymath}
where the~$y_i$'s are labels in~$\{-1,+1\}$, the~$\x_i$'s are gene expression vectors
in~$\Real^p$.  The weight~$n_1$ is the number of positive samples,
whereas~$n_{-1}$ the number of negative ones.
This model does not include an intercept, but the gene expressions are
centered. The regularization functions included in the comparison are the following:
\begin{myitemize}
\item our path-coding penalties~$\varphip$ and~$\psip$ with the weights~$\eta_g$ defined in~(\ref{eq:choice});
\item the squared~$\ell_2$-penalty (ridge logistic regression);
\item the~$\ell_1$-norm (sparse logistic regression);
\item the elastic-net penalty of~\citet{zou}, which has the form~$\w \to \|\w\|_1 + (\mu/2)\|\w\|_2^2$, where~$\mu$ is an additional parameter;
\item the penalty~$\psi$ of~\citet{jacob} where the groups~$\G$ are pairs of vertices linked by an~arc; 
\item a variant of the penalty~$\psi$ of~\citet{jacob} whose form is given in Equation~(\ref{eq:convex2}) of Appendix~\ref{appendix:links}, where the~$\ell_2$-norm is used in place of the~$\ell_\infty$-norm;
\item the penalty~$\zeta_\G$ of~\citet{jenatton} given in Appendix~\ref{appendix:rodolphe} where the groups are all pairs of vertices linked by an arc;
\item the penalty~$\zeta_\G$ of~\citet{jenatton} using the group structure
adapted to DAGs described in Appendix~\ref{appendix:rodolphe}. This penalty has shown to be
empirically problematic to use directly. The number of groups
each variable belongs to significantly varies from a variable to another, resulting in
overpenalization for some variables and underpenalization for some others. To
cope with this issue, we have tried different strategies to choose
the weights~$\eta_g$ for every group in the penalty, similarly as those
described by~\citealt{jenatton}, but we have been unable to obtain sparse
solutions in this setting (typically the penalty selects here more than a
thousand variables). A heuristic that has proven to be much better is to
add a weighted~$\ell_1$-penalty to~$\zeta_\G$ to correct the
over/under-penalization issue. Denoting for a variable~$j$ in~$\{1,\ldots,p\}$ by~$d_j$ the number
of groups the variable~$j$ belongs to---in other words $d_j \defin \sum_{g \in \G : g \ni j} 1$, we
add the term $\sum_{j=1}^p (\max_k d_k - d_j)|\w_j|$ to the penalty~$\zeta_\G$.
\end{myitemize}
The parameter~$\lambda$ in Eq.~(\ref{eq:prob}) is chosen on a logarithmic scale
with steps $2^{1/4}$.  The elastic-net parameter~$\mu$ is chosen
in~$\{1,10,100\}$.  The parameter~$\gamma$ for the penalties~$\varphip$
and~$\psip$ is chosen in~$\{2,4,8,16\}$.
We proceed by randomly sampling~$20\%$ of the data as a test set,
keeping~$80\%$ for training, selecting the parameters $\lambda,\mu,\gamma$
using internal~$5$-fold cross-validation on the training set, and we measure
the average balanced error rate between the two classes on the test set.  We have
repeated this experiment~$20$ times and we report the averaged results in
Table~\ref{table:breastcancer}. 

We start by answering question~\textbf{(A)}. We remark that our
penalties~$\varphip$ and~$\psip$ succeed in selecting very few connected
components of~$G_0$, on average~$1.3$ for~$\psip$ and~$1.6$ for $\varphip$
while providing sparse solutions. This significantly improves the connectivity
of the solutions obtained using the approach of~\citet{jacob}
or~\citet{jenatton}.  To claim interpretable results, one has of course
to trust the original graph. Like~\citet{jacob}, we assume that connectivity
in~$G_0$ is good a priori.  We also study the effect of the
preprocessing step we have used to obtain a directed acyclic graph~$G$
from~$G_0$.  We report in the row ``$\psip$-random'' in
Table~\ref{table:breastcancer} the results obtained when randomizing the
pre-processing step between every experimental run (providing us a different
graph $G$ for every run). We observe that the outcome~$G$ does not significantly
change the sparsity and connectivity in~$G_0$ of the sparsity patterns our penalty
selects.

As far as the prediction performance is concerned, $\psip$ seems to
be the only penalty that is able to produce sparse and connected solutions while
providing a similar average error rate as the~$\ell_2$-penalty. The non-convex
penalty~$\varphip$ produces a very sparse solution which is connected as well,
but with an approximately $6\%$ higher classification error rate. Note that
because of the high variability of this performance measure, clearly assessing
the statistical significance of the observed difference is difficult. As it was
previously observed by~\citet{jacob}, the data is very noisy and the number of
samples is small, resulting in high variability. As~\citet{jacob}, we have been
unable to test the statistical significance rigorously---that is, without
assuming independence of the different experimental runs.  We can
therefore not clearly answer the first part of question~\textbf{(B)}.  The
second part of the question is however clearer:
neither sparsity, nor the graph structure seems to help prediction in this
experiment. We have for example tried to use the same graph~$G$, but where we
randomly permute the~$p$ predictors (genes) at every run, making the graph
structure irrelevant to the data. We report in Table~\ref{table:breastcancer}
at the row~``$\psip$-permute'' the average classification error rate, which is
not significantly different than without permutation.

Our conclusions about the use of structured sparse estimation for this task are
therefore mixed. On the one hand, none of the tested method are shown to
perform statistically better in prediction than simple ridge regularization.
On the other hand, our penalty~$\psip$ is the only one that performs as well
as ridge while selecting a few predictive genes forming a a well-connected
sparsity pattern. According to~\citet{jacob}, this is a significant asset
for biologists, assuming the original graph should be trusted. 

Another aspect we would like to study is the
stability properties of the selected sparsity patterns, which is often an issue
with features selection methods~\citep{meinshausen}. By introducing strong
prior knowledge in the regularization, structured sparse estimation seems to
provide more stable solutions than $\ell_1$. For instance, $5$ genes are
selected by~$\ell_1$ in more than half of the experimental runs, whereas this
number is $10$ and~$14$ for the penalties of~\citet{jacob}, and~$33$
for~$\psip$. Whereas we believe that stability is important,
it is however hard to claim direct benefits of having a ``stable''
penalty without further study. By encouraging connectivity of the solution,
variables that are highly connected in the graph tend to be more often
selected, improving the stability of the solution, but not necessarily its
interpretation in the absence of biological prior knowledge that prefers
connectivity.

\begin{table}[hbtp]
\centering
\begin{tabular}{|c||c|c|c|c|}
\hline
   &  test error (\%) & sparsity & connected components \\
\hline
\hline
$\ell_2^2$ &    $31.0 \pm 6.1$  &  $7910$  & $58$   \\
\hline
$\ell_1$ &    $36.0 \pm 6.5$  &  $32.6$  & $30.9$    \\
\hline
$\ell_2^2+\ell_1$ &    $31.5 \pm 6.7$  &  $929.6$  & $355.2$   \\
\hline
\citet{jacob}-$\ell_\infty$ &    $35.9 \pm 6.8$  &  $68.4$  & $13.2$   \\
\hline
\citet{jacob}-$\ell_2$ & $36.0 \pm 7.2 $    & $58.5$  & $11.1$   \\
\hline
\citet{jenatton} (pairs) &  $34.5 \pm 5.2$  &  $33.4$  & $28.8$   \\
\hline
\citet{jenatton} (DAG)+weighted~$\ell_1$ & $35.6 \pm 7.0$  & $54.6$ & $28.4$ \\
\hline
\hline
$\varphip$ & $36.0 \pm 6.8$  & $10.2$  & $1.6$  \\
\hline
$\psip$ & $30.2 \pm 6.8$  & $69.9$  & $1.3$ \\
\hline
\hline
$\psip$-permute & $33.2 \pm 7.6$  & $143.2$  & $1.7$  \\
\hline
$\psip$-random & $31.6 \pm 6.0$  & $78.5$  & $1.4$  \\
\hline
\end{tabular}
\caption{Experimental results on the breast cancer dataset. Column ``{test error}'': average balanced classification error rate on the test set in percents with standard deviations; the results are averaged over~$20$ runs and the parameters are selected for each run by internal $5$-fold cross-validation. Column~``{sparsity}'': average number of selected genes. Column~``{connected components}'': average number of selected connected components in~$G_0$.}\label{table:breastcancer}
\end{table}

To answer question \textbf{(C)}, we study the computational efficiency of our implementation.
One iteration of the proximal gradient method for the selected parameters is
relatively fast, approximately $0.17$s for $\psip$ and $0.15$s for $\varphip$  on a $1.2$GHz laptop
CPU (Intel core i3 330UM), but it tends to be significantly slower when the
solution is less sparse, for instance with small values for~$\lambda$.  Since
solving an instance of problem~(\ref{eq:prob}) requires computing about~$500$
proximal operators to obtain a reasonably precise solution, our method is fast
enough to conduct this experiment in a reasonable amount of time.
Of course, simpler penalties are faster to use. An iteration of the proximal
gradient method takes about $0.15$s for $\zeta_\GG$, $0.05$s for \citet{jacob},
$0.01$s for $\ell_2$ and~$0.003$s for $\ell_1$.

%% file: ccl.tex
Our paper proposes a new form of structured penalty for supervised learning
problems where predicting features are sitting on a DAG, and where one wishes
to automatically select a few connected subgraphs of the DAG.  The
computational feasibility of this form of penalty is established by making a
new link between supervised path selection problems and network flows. Our
penalties admit non-convex and convex variants, which can be used within the
same network flow optimization framework. These penalties are flexible in the
sense that they can control the connectivity of a problem solution, whether one
wishes to encourage large or small connected components, 
and are able to model long-range interactions between variables.

Some of our conclusions show that being able to provide both non-convex and
convex variants of the penalties is valuable. In various contexts,
we have been able to find situations where convexity is helpful, and others
where the non-convex approach leads to better solutions than the convex
one. Our experiments show that when connectivity of a sparsity pattern is a
good prior knowledge our approach is fast and effective for solving different
prediction problems. 

Interestingly, our penalties seem to perform empirically well on more general
graphs than DAGs, when heuristically removing cycles, and we would like in the
future to find a way to better handle~them. We also would like to make further
connections with image segmentation techniques, which exploit different but
related optimization techniques~\citep[see][]{boykov,couprie}, and kernel
methods, where other types of feature selection in DAGs occur~\citep{bach7}. 

Finally, we are also interested in applying our techniques to sparse estimation
problems where the sparsity pattern is expected to be exactly a combination of
a few paths in a DAG. While the first version of this manuscript was under
review, the first author started a collaboration with computational biologists
to address the problem of isoform detection in RNA-Seq data. In a nutshell, a
mixture of small fragments of mRNA is observed and the goal is to find a few
mRNA molecules that explain the observed mixture. In mathematical terms, it
corresponds to selecting a few paths in a directed acyclic graph in a penalized
maximum likelihood formulation. Preliminary results obtained by~\citet{bernard}
confirm that one could achieve state-of-the-art results for this task by
adapting part of the methodology we have developed in this paper.

%% file: appendix.tex
\section{The Penalty of~\citet{jenatton} for DAGs} \label{appendix:rodolphe}
The penalty of~\citet{jenatton} requires a pre-defined set of possibly
overlapping groups~$\GG$ and is defined as follows:
\begin{equation}
\zeta_\GG(\w) \defin \sum_{g \in {\mathcal G}} \eta_g \|\w_g\|_\nu,
\label{eq:group} 
\end{equation} 
where the vector $\w_g$ in $\Real^{|g|}$ records the coefficients of $\w$
indexed by $g$ in $\GG$, the scalars $\eta_g$ are positive weights,
and~$\nu$ typically equals~$2$ or~$\infty$. This penalty can be
interpreted as the~$\ell_1$-norm of the vector $[\eta_g
\|\w_g\|_\nu]_{g \in \GG}$, therefore inducing sparsity at the group
level. It extends the Group Lasso~\citep{turlach,yuan,zhao} by allowing
the groups to overlap.

Whereas the penalty~$\psi$ of~\citet{jacob} encourages solutions whose set
of non-zero coefficients is a \emph{union} of a few groups, the
penalty~$\zeta_\GG$ promotes solutions whose sparsity pattern is in the
\emph{intersection} of some selected groups.  This subtlety
makes these two lines of work significantly different.
It is for example unnatural to use the penalty~$\zeta_\GG$ to encourage
connectivity in a graph. When the groups are defined as the pairs of vertices
linked by an arc, it is indeed not clear that sparsity patterns defined as the
intersection of such groups would lead to a well-connected subgraph. As shown
experimentally in Section~\ref{sec:exp}, this setting indeed performs poorly for
this task.

However, when the graph is a DAG, there exists an appropriate group setting~$\GG$
\emph{when the sparsity pattern of the solution is expected to be a single
connected component of the DAG}. Let us indeed define the groups to be the sets
of ancestors, and sets of descendents for every vertex; the set of
descendents of a vertex~$u$ in a DAG are defined as all vertices~$v$ such
that there exists a path from~$u$ to~$v$. Similarly the set of ancestors
contains all vertices such that there is a path from~$v$ to~$u$.
The corresponding penalty~$\zeta_\GG$ encourages sparsity patterns
which are intersections of groups in~$\GG$, which can be shown to be
exactly the connected subgraphs of the DAG.\footnote{This setting was suggested
to us by Francis Bach, Rodolphe Jenatton and Guillaume Obozinski in a
private discussion. Note that we have assumed here for simplicity that
the DAG is connected---that is, has a single connected component.} This penalty
is tractable since the number of groups is linear in the number of vertices,
but as soon as the sparsity pattern of the
solution is not connex (contains more than one connected component), it is
unable to recover it, making it useful to seek for a more flexible approach.
For this group structure~$\GG$, the penalty~$\zeta_\GG$ also suffers from other
practical issues concerning the overpenalization of variables belonging to many
different groups. These issues are empirically discussed in
Section~\ref{sec:exp} on concrete examples.

Interestingly, \citet{mairal11} have shown that the penalty~$\zeta_\GG$ with
$\nu=\infty$ and any arbitrary group structure $\GG$ is related to network
flows, but for different reasons than the penalties~$\varphip$ and~$\psip$. The
penalty $\zeta_\GG$ is indeed unrelated to the concept of graph sparsity since
it does not require the features to have any graph structure. Solving
regularized problems with~$\zeta_\GG$ is however challenging, and
\citet{mairal11} have shown that the proximal operator of~$\zeta_\GG$ could be
computed by means of a parametric maximum flow formulation. It involves a
bipartite graph where the nodes represent variables and groups, and arcs
represent inclusion relations between a variable and a group. \citet{mairal11}
address thus a significantly different problem than ours, even though there is
a common terminology between their work and ours.

\section{Links Between~\citet{huang} and~\citet{jacob}} \label{appendix:links}
Similarly as the penalty of~$\varphi$ of~\citet{huang}, the penalty of~\citet{jacob}
encourages the sparsity pattern of a solution to be the union of a small number
of predefined groups~$\GG$.  Unlike the function~$\varphi$, it is convex (it can be shown to be a norm), and is
defined as follows: 
\begin{equation} 
\psi'(\w) \defin
\min_{(\xib^g \in \Real^p)_{g \in \GG}} \left\{ \sum_{g \in \GG} \eta_g
\|\xib^g\|_\nu  \st \w = \sum_{g \in \GG} \xib^g ~~~\text{and}~~~ \forall~g \in
\GG,~ \text{Supp}(\xib^g) \subseteq g \right\}, \label{eq:convex2}
\end{equation} 
where~$\|.\|_\nu$ typically denotes the~$\ell_2$-norm
($\nu\!=\!2$) or~$\ell_\infty$-norm ($\nu\!=\!\infty$).   In this equation, the
vector~$\w$ is decomposed into a sum of latent vectors $\xib^g$, one for every
group~$g$ in~$\GG$, with the constraint that the support of~$\xib^g$ is itself
in~$g$.  The objective function is a group Lasso penalty~\citep{yuan,turlach}
as presented in Equation~(\ref{eq:group}) which encourages the vectors~$\xib^g$
to be zero. As a consequence, the support of~$\w$ is contained in the union of
a few groups~$g$ corresponding to non-zero vectors $\xib^g$, which is exactly
the desired regularization effect.
We now give a proof of Lemma~\ref{lemma:equiv} relating this penalty to the
convex relaxation of~$\varphi$ given in Equation~(\ref{eq:convex}) when~$\nu=\infty$.

\begin{proof}
We start by showing that $\psi'$ is equal to the penalty $\psi$ defined in Equation~(\ref{eq:convex}) on $\Real_+^p$.
We consider a vector~$\w$ in~$\Real_+^p$ and introduce for all groups $g$ in $\GG$ appropriate variables $\xib^g$ in $\Real^p$. The linear program defining~$\psi$ can be equivalently rewritten
\begin{displaymath}
   \psi(\w) =  \min_{\substack{  \x \in \Real_+^{|\GG|} \\ (\xib^g \in \Real^p)_{g \in \GG}}} \left\{ \eta^\top \x  \st \sum_{g \in \GG} \xib^g = \w,~ \NN\x \geq \sum_{g\in\GG}\xib^g ~\text{and}~  \forall~g \in \GG,~ \text{Supp}(\xib^g) \subseteq g\right\},
\end{displaymath}
where we use the assumption that for all vector $\w$ in $\Real_+^p$, there exist vectors $\xib^g$ such that $\sum_{g \in \GG} \xib^g = \w$.
Let us consider an optimal pair $(\x,(\xib^g)_{g\in\GG})$. For all indices $j$ in $\{1,\ldots,p\}$, the constraint~$\NN\x \geq \sum_{g\in\GG}\xib^g$ leads to the following inequality 
$$
\underbrace{\sum_{g \ni j : x_g \geq \xib^g_j} x_g- \xib^g_j}_{\tau_j^+ \geq 0}  + \underbrace{\sum_{g \ni j : x_g < \xib^g_j} x_g- \xib^g_j}_{\tau_j^- \leq 0}  \geq 0,
$$ 
where $x_g$ denotes the entry of $\x$ corresponding to the group $g$, and
two new quantities~$\tau_j^+$ and~$\tau_j^-$ are defined. 
For all $g$ in $\GG$, we define a new vector $\xib^{\prime g}$ such that for every pair $(g,j)$ in $\GG \times \{1,\ldots,p\}$:
\begin{enumerate}
 \item if $j \notin g$, $\xib^{\prime g}_j \defin 0$; 
  \item if $j \in g$ and $x_g \geq \xib^g_j$, then $\xib^{\prime g}_j \defin x_g$;
  \item if $j \in g$ and $x_g < \xib^g_j$, then $\xib^{\prime g}_j \defin \xib_j^g - (x_g-\xib_j^g)\frac{\tau_j^+}{\tau_j^-}$. 
\end{enumerate}
Note that if there exists~$j$ and~$g$ such that~$x_g < \xib^g_j$, then~$\tau_j^-$ is nonzero and the quantity~$\tau_j^+ / \tau_j^-$ is well defined.
Simple verifications show that for all indices $j$ in~$\{1,\ldots,p\}$, we have $\sum_{g \ni j} x_g- \xib^{\prime g}_j= \tau_j^+ +\tau_j^- =\sum_{g \ni j} x_g-\xib^g_j$, and therefore $\sum_{g \in \GG} \xib^{\prime g} = \sum_{g \in \GG} \xib^{g} = \w$.
The pair $(\x,(\xib^{\prime g})_{g\in\GG})$ is therefore also optimal.
In addition, for all groups $g$ in $\GG$ and index~$j$ in~$\{1,\ldots,p\}$, it is easy to show that~$x_g - \xib_j^{\prime g} \geq 0$ and that we have at optimality $\sign(\xib^g_j) = \sign(\w_j) = 1$ for any nonzero~$\xib^g_j$. Therefore, the condition
$\|\xib^{\prime g}\|_\infty \leq x_g$ is satisfied, which is stronger than the original
constraint $\NN\x \geq \sum_{g \in \GG} \xib^{\prime g}$.
Moreover, it is easy to show that $\|\xib^{\prime g}\|_\infty$ is necessary
equal to $x_g$ at optimality (otherwise, one could decrease the value of~$x_g$
to decrease the value of the objective function).  We can now
rewrite $\psi(\w)$ as
\begin{displaymath}
   \psi(\w) = \left\{ \min_{  (\xib^g \in \Real^p)_{g \in \GG} } \sum_{g\in\GG} \eta_g \|\xib^g\|_\infty  \st \sum_{g \in \GG} \xib^g = \w, ~\text{and}~ \forall~g \in \GG,~ \text{Supp}(\xib^g) \subseteq g \right\},
\end{displaymath}
and we have shown that $\psi'=\psi$ on $\Real_+^p$.
By noticing that in Equation~(\ref{eq:convex}) a solution $(\xib^g)_{g\in\GG}$ necessary satisfies $\sign(\xib^g_j) = \sign(\w_j)$ for every group $g$ and index $j$ such that $\xib^g_j \neq 0$, we can extend the proof from~$\Real_+^p$ to~$\Real^p$.
\end{proof}

\section{Interpretation of the Weights~$\eta_g$ with Coding Lengths} \label{appendix:coding}
\citet{huang} have given an interpretation of the penalty~$\varphi$ defined in 
Equation~(\ref{eq:nonconvex}) in terms of coding length. We use similar arguments
to interpret the path-coding penalty~$\varphip$ from an information-theoretic point of view.
For appropriate weights~$\eta_g$,
the quantity~$\varphip(\w)$ for a vector~$\w$ in~$\Real^p$ can be seen as
a coding length for the sparsity pattern of~$\w$---that is, the
following Kraft-MacMillan inequality~\cite[see][]{cover,mackay} is satisfied:
\begin{displaymath}
    \sum_{S \in \{0,1\}^p} 2^{-\varphip(S)} \leq 1.
\end{displaymath}
It is indeed well known in the information theory literature that there exists a binary
uniquely decodeable code over $\{0,1\}^p$ with code length $\varphip(S)$ for
every pattern $S$ in $\{0,1\}^p$ if and only if the above inequality is
satisfied~\cite[see][]{cover}.
We now show that a particular choice for the weights~$\eta_g$ leads
to an interesting interpretation.

Let us consider the graph $G'$ with source and sink vertices~$s$ and~$t$
defined in Section~\ref{sec:approach}.  We assume that a probability matrix
transition $\pi(u,v)$ for all~$(u,v)$ in~$E'$ is given.
With such a matrix transition, it is easy to obtain a coding length for the set
of paths~$\GG_p$:
 \begin{lemma}[Coding Length for Paths.]\label{lemma:coding}~\newline
    Let $\text{cl}_g$ for a path $g=(v_1,\ldots,v_k)$ in $\GG_p$ be defined as
    \begin{displaymath}
       \text{cl}_g \defin -\log_2 \pi(s,v_1) - \Big(\sum_{i=1}^{k-1}\log_2 \pi(v_i,v_{i+1}) \Big) - \log_2\pi(v_k,t).
    \end{displaymath}
    Then $\text{cl}_g$ is a coding length on $\GG_p$.
 \end{lemma}
 \begin{proof}
 We observe that for every path $(v_1,\ldots,v_k)$ in $\GG_p$ corresponds a unique
 walk of length $|V'|$ of the form $(s,v_1,\ldots,v_b,t,t,\ldots,t)$,
 and vice versa.
 Denoting by $\pi^t(s,t)$ the probability that a Markov chain associated to the probability transition matrix $\pi$ starting at the vertex $u$ is
 at the vertex $v$ at time~$t$, it is easy to show that 
 \begin{displaymath}
    \sum_{g \in \GG_p} 2^{-\text{cl}_g} = \pi^{|V'|}(s,t) = 1,
 \end{displaymath}
 and therefore $\text{cl}_g$ is a coding length on $\GG_p$.
 \end{proof}
the term $-\log_2 \pi(s,v_1)$ represents the number of bits used to indicate
that a path~$g$ starts with the vertex~$v_1$, whereas the bits corresponding to
the terms $-\log_2 \pi(v_i,v_{i+1})$ indicate that the vertex following~$v_i$
is~$v_{i+1}$.  The bits corresponding to last term $-\log_2 \pi(v_k,t)$
indicate the end of the path.
To define the weights~$\eta_g$, we now define the following costs:
$$
   c_{uv}\defin \left\{ \begin{array}{rl} 1 - \log_2 \pi(u,v) & ~\text{if}~u=s \\ -\log_2 \pi(u,v) & ~\text{otherwise.} \end{array}\right.
$$
The weight~$\eta_g$ therefore satisfies $\eta_g= \sum_{(u,v) \in E'} c_{uv} =
\text{cl}_g + 1$, and as shown by~\citet{huang}, this is a sufficient condition
for $\varphip(\w)$ to be a coding length for $\{0,1\}^p$.

We have therefore shown that (i) the different terms composing the
weights~$\eta_g$ can be interpreted as the number of bits used to encode the
paths in the graph; (ii) it is possible to use probability transition matrices
(or random walks) on the graph to design the weights~$\eta_g$.

\section{Proofs of the Propositions} \label{appendix:proofs}
In this section, we provide the proofs of our main results.
\subsection{Proofs of Propositions~\ref{prop:varphi} and~\ref{prop:psi}}
 \begin{proof}
 We start by proving Proposition~\ref{prop:varphi}.
 Let us consider the alternative definition of $\varphip$ given in
 Equation~(\ref{eq:altphi}). This is an optimization problem over all paths in $G$, or
 equivalently all $(s,t)$-paths in~$G'$ (since these two sets are in bijection).
 We associate to a vector $\x$ in $\{0,1\}^p$ a flow $f$ on $G'$, obtained by
 sending one unit of flow on every $(s,t)$-path $g$ satisfying $x^g = 1$ ($x^g$
 denotes the entry of~$\x$ associated to the group/path~$g$). Each of these
 $(s,t)$-path flow has a cost $\eta_g$ and  
 the total cost of~$f$ is exactly~$\eta^\top\x$. 
 
 We also observe that within this network flow framework, the constraint $\NN\x
 \geq \text{Supp}(\w)$ in Equation~(\ref{eq:altphi}) is equivalent to saying that
 for all $j$ in $\{1,\ldots,p\}$ the amount of flow going through the
 vertex~$j$ (denoted by $s_j(f)$) is greater than one if $\w_j \neq 0$.  We have
 therefore shown that $\varphip(\w)$ is the minimum cost achievable by a flow $f$
 such that the constraint $s_j(f) \geq 1$ is satisfied for all $j$ in
 $\text{Supp}(\w)$ and such that $f$ can be decomposed into binary $(s,t)$-path
 flows.
 
 To conclude the proof of Proposition~\ref{prop:varphi}, we show that there
 exists an optimal flow that admits a decomposition into binary $(s,t)$-path
 flows.  We notice that all arc capacities in Equation~(\ref{eq:phiflow}) are
 integers.  A classical result~\citep[][Theorem 9.10]{ahuja} says that there
 exists an optimal integer minimum-cost flow (a flow whose values on arcs are
 integers).  We denote by $f$ such a solution.  Then, the flow decomposition
 theorem~\citep[][Proposition 1.1]{bertsekas2} tells us that $f$ can be
 decomposed into $(s,t)$-path flows, but it also says that if $f$ is integer,
 then $f$ can be decomposed into integer  $(s,t)$-path flows. We conclude
 the proof by noticing that sending more than one unit of flow on a path is not
 optimal (one can reduce the cost by sending only one unit of flow, while
 keeping the capacity constraints satisfied), and therefore there exists in fact
 a decomposition of $f$ into binary $(s,t)$-path flows.
 The quantity presented in Equation~(\ref{eq:phiflow}) is therefore equal to $\varphip(\w)$.

The proof of Proposition~\ref{prop:psi} builds upon the definition
of $\psi$ given in Equation~(\ref{eq:convex}) and is
similar to the one of Proposition~\ref{prop:varphi}.
\end{proof}
\subsection{Proof of Proposition~\ref{prop:proxphi}}
\begin{proof}
Using the definition of the proximal operator in Equation~(\ref{eq:prox_problem})
and the definition of $\varphi$ in Equation~(\ref{eq:altphi}), 
there exists a pattern $S$ in $\{0,1\}^p$ such that the solution
$\w^\star$ of the proximal problem satisfies for all~$j$, $\w^\star_j=\u_j$ if $j$ is in $S$,
and $\w^\star_j=0$ otherwise.
We therefore rewrite Equation~(\ref{eq:prox_problem}) by using the result of Proposition~\ref{prop:varphi}
\begin{displaymath}
    \min_{S \in \{0,1\}^p, f \in \FF}  \left\{ \frac{1}{2} \sum_{j \notin S} \u_j^2 + \sum_{(u,v) \in E'} f_{uv}c_{uv} \st s_j(f) \geq 1, \forall j \in S \right\}.
 \end{displaymath}
 When $S$ is fixed, the above expression is a minimum cost flow
 problem with integer capacity constraints. Thus, there exists an integer flow
 solution, and we can, without loss of generality, constrain $f$ to be integer,
 and replace the constraints $s_j(f) \geq 1$ by $s_j(f) > 0$.
 After this modification, for $f$ is fixed, minimizing with respect to $S$ gives
 us the following closed form: for all $j$ in $\{1,\ldots,p\}$, $S_j=1$ if $s_j(f) > 0$
 and $0$ otherwise. With this choice for $S$, we have in addition $\sum_{j \notin S} \u_j^2 = \sum_{j=1}^p \max\big(\u_j^2(1-s_j(f)),0\big)$,
 and denoting by $\FF_{\text{int}}$ the set of integer flows, we can equivalently rewrite the optimization problem
 \begin{displaymath}
    \min_{f \in \FF_{\text{int}}}  \left\{\sum_{(u,v) \in E'} f_{uv}c_{uv} + \sum_{j=1}^p \frac{1}{2}\max\big(\u_j^2(1-s_j(f)),0\big)\right\}.
 \end{displaymath}
 It is easy to transform this minimum cost flow problem with piecewise linear
 costs to a classical minimum cost flow problem~\citep[see][Exercise
 1.19]{bertsekas2} with integral constraints.  Therefore, it is possible to
 remove the constraint $f \in \FF_{\text{int}}$ and replace it by $f \in \FF$
 without changing the optimal value of the cost function, leading to the
 formulation proposed in Equation~(\ref{eq:proxphi}).
 \end{proof}
 \subsection{Proof of Proposition~\ref{prop:proxpsi}}
 \begin{proof}
  Without loss of generality, let us suppose that $\u$ is in
 $\Real_+^p$. Let us indeed denote by $\w^\star \defin \text{Prox}_{\psip}[\u]$.
 It is indeed easy to see that the signs of the entries of $\w^\star$ are
 necessary the same as those of $\u$, and flipping the signs of some entries of
 $\u$ results in flipping the signs of the corresponding entries in~$\w^\star$.
 According to Proposition~\ref{prop:psi}, we can write the proximal problem as
 \begin{displaymath}
    \min_{\w \in \Real_+^p, f \in \FF}  \left\{ \frac{1}{2} \sum_{j=1}^p (\u_j-\w_j)^2 + \sum_{(u,v) \in E'} f_{uv}c_{uv} \st s_j(f) \geq \w_j, \forall j \in \{1,\ldots,p\} \right\}.
 \end{displaymath}
 When $f$ is fixed, minimizing with respect to $\w$ yields for all $j$, $\w_j^\star=\min(\u_j,s_j(f^\star))$. Plugging this closed form in the above equation yields the desired formulation.
 \end{proof}
\subsection{Proof of Proposition~\ref{prop:psistar}}
 \begin{proof}
 We recall that according to Lemma~\ref{lemma:equiv} we have for all $\w$ in $\Real_+^p$
\begin{displaymath}
   \psip(\w) = \min_{  \x \in \Real_+^{|\GG_p|}} \left\{ \eta^\top \x  \st \NN\x \geq \w  \right\}. 
\end{displaymath}
This is a linear program, whose dual~\citep[see][]{nocedal} gives us another definition for~$\psip$ on $\Real_+^p$.
Since strong duality holds here, we have
\begin{displaymath}
   \psip(\w) = \max_{  \kappab \in \Real_+^{p}} \left\{ \w^\top \kappab \st \NN^\top\kappab \leq \eta  \right\}. 
\end{displaymath}
It is easy to show that one can extend this definition on $\Real^p$ such that we have
\begin{equation}
   \psip(\w) = \max_{  \kappab \in \Real^{p}} \left\{ \w^\top \kappab \st \max_{g \in \G_p} \frac{\|\kappab_g\|_1}{\eta_g} \leq 1    \right\},  \label{eq:dualpsi}
\end{equation}
where $\kappab_g$ denotes the vector of size $|g|$ containing the entries of
$\kappab$ corresponding to the indices in the group $g$. Note that a similar
formula appears in~\citep[][Lemma 2]{jacob}, when the $\ell_2$-norm is used in
place of the $\ell_\infty$.
We now define for a vector $\kappab$ in $\Real^p$,
\begin{displaymath}
   \psip^*(\kappab) \defin \max_{g \in \G_p} \frac{\|\kappab_g\|_1}{\eta_g}.
\end{displaymath}
It is easy to see that it is a norm, and by Equation~(\ref{eq:dualpsi}), this is in
fact the dual norm of the norm~$\psip$.
We can now rewrite it as
\begin{displaymath}
\begin{split}
   \psip^*(\kappab) & = \min_{\tau \in \Real_+} \left\{ \tau \st \max_{g \in \G_p} \frac{\|\kappab_g\|_1}{\eta_g} \leq \tau \right\}, \\
        & = \min_{\tau \in \Real_+} \left\{ \tau \st \max_{g \in \G_p} \frac{\|\kappab_g\|_1}{\tau} - \eta_g \leq 0 \right\}, \\
        & = \min_{\tau \in \Real_+} \left\{ \tau \st \min_{g \in \G_p} l_\tau(g) \geq 0 \right\}, \\
\end{split}
\end{displaymath}
where we have identified the groups in $\GG_p$ to their corresponding $(s,t)$-paths in $G'$. 
\end{proof}
\subsection{Proof of Proposition~\ref{prop:algphi}}
\begin{proof} \\
\textbf{Correctness:} \\
We start by showing that when the algorithm converges, it returns the correct
solution.  We remark that the choice of~$\tau$ in the algorithm ensures that
there always exists a group~$h$ in~$\GG_p$ such that~$l_\tau(h)=0$ and therefore
we always have~$\delta \leq 0$. Thus, when the algorithm converges,
$\delta$ is equal to zero.  Moreover, the function $G: \tau \to \min_{h \in
\GG_p} l_\tau(h)$ is non-increasing with~$\tau$ since the functions~$\tau \to
l_\tau(h)$ are themselves non-increasing for all~$h$ in~$\GG_p$. It is also easy
to show that there exists a unique~$\tau$ such that~$G(\tau)=0$, which is the
desired solution.  We conclude by noticing that at convergence, we have
$G(\tau)=\delta=0$.\\
\textbf{Convergence and complexity:} \\
We now show that the algorithm converges and give a worst-case complexity.
We denote by~$\tau_k$, $g_k$ and~$\delta_k$
the respective values of~$\tau, g$ and~$\delta$ at the iteration~$k$ of the
algorithm. The definition of~$\tau_{k+1}$ implies that
\begin{displaymath}
   l_{\tau_{k+1}}(g_k) = 0 = \underbrace{l_{\tau_{k}}(g_k)}_{\delta_k \leq 0} + \underbrace{\|\kappab_{g_k}\|_1\Big(\frac{1}{\tau_k}-\frac{1}{\tau_{k+1}}\Big)}_{-\delta_k \geq 0}.
\end{displaymath}
Moreover,
\begin{displaymath}
   \delta_{k+1} = l_{\tau_{k+1}}(g_{k+1}) = l_{\tau_{k}}(g_{k+1}) + \|\kappab_{g_{k+1}}\|_1\Big(\frac{1}{\tau_k}-\frac{1}{\tau_{k+1}}\Big).
\end{displaymath}
Since $l_{\tau_{k}}(g_{k+1}) \geq \delta_k$ ($\delta_k$ is the length of the shortest path), we can show that
\begin{displaymath}
   \delta_{k+1} \geq \delta_k\Big(1  - \frac{\|\kappab_{g_{k+1}}\|_1}{\|\kappab_{g_{k}}\|_1}\Big).
\end{displaymath}
Since $\delta_{k+1} \leq 0$, we remark that necessarily $\|\kappab_{g_{k+1}}\|_1 \leq \|\kappab_{g_{k}}\|_1$, and we have two possibilities
\begin{enumerate}
   \item either $\|\kappab_{g_{k+1}}\|_1=\|\kappab_{g_{k}}\|_1$ and $\delta_{k+1}=0$, meaning that the algorithm has converged.
   \item either $\|\kappab_{g_{k+1}}\|_1<\|\kappab_{g_{k}}\|_1$ and it is easy to show that is implies that $\eta_{g_{k+1}} < \eta_{g_k}$.
\end{enumerate}
Since~$\eta_h = \gamma + |h|$, we obtain that $\eta_{g_k}$ is strictly decreasing with~$k$ before the convergence of the algorithm. Since it can
have at most~$p$ different values, the algorithm converges in at most~$p$ iterations.
Updating the path~$g$ in the algorithm can be done by solving a shortest path problem in the graph~$G'$, which can be done in $O(|E|)$ operations since the graph is acyclic~\citep{ahuja}, and the total worst-case complexity is $O(p|E|)$, which concludes the proof.
\end{proof}
\subsection{Proof of Proposition~\ref{prop:activeset}}
\begin{proof}
We denote by~$\kappab$ the quantity~$\kappab\defin\nabla L(\w)$, and
respectively by~$\tilde{\kappab}$ and~$\tilde{\w}$ the vectors recording the
entries of~$\kappab$ and~$\w$ that are in~$\tilde{V}$.\\
\textbf{Convergence of the algorithm:} \\
Convergence of the algorithm is easy to show and consists of observing
that~$\tilde{G}$ is strictly increasing. After solving subproblem~(\ref{eq:subprob}), we have from the optimality
conditions of Lemma~\ref{lemma:opt} that
${\oldpsi}_{\tilde{\GG}_p}^*(\tilde{\kappab}) \leq \lambda$.
By definition of the dual-norm given in Proposition~\ref{prop:psistar}, and using the same notation, we
have that for all~$g$ in~$\tilde{\GG_p}$, $l_\lambda(g) \geq 0$.
We now denote by~$\tau$ the quantity~$\tau = \psip^*(\kappab)$; if~$\tau \leq \lambda$, the algorithm stops. If not,
we have that for all~$g$ in~$\tilde{\GG_p}$, $l_\tau(g) > 0$ (since~$\tau > \lambda$ and $l_\lambda(g) \geq 0$ for all~$g$ in~$\tilde{\GG_p}$).
The step $g \leftarrow \argmin_{g \in \GG_p} l_\tau(g)$ then selects a group~$g$
such that~$l_\tau(g)=0$ (which is easy to check given the definition
of~$\psip^*$ in Proposition~\ref{prop:psistar}. Therefore, the selected path~$g$ is not
in~$\tilde{G}$, and the size of $\tilde{G}$ strictly increases, which guarantees the convergence of the algorithm.\\
\textbf{Correctness:} \\
We want to show that when the algorithm stops, it returns the correct solution.
First, if we have~$\tilde{G}=G$, it is trivially correct. If it stops with~$\tilde{G}\neq G$, we have that $\psip^*(\kappab)  \leq \lambda$, and according to Lemma~\ref{lemma:opt}, we only need to check that~$-\kappab^\top\w= \lambda \psip(\w)$.
We remark that we have
$\lambda \psip(\w) \leq \lambda {\oldpsi}_{\tilde{\GG}_p}(\w) =
-\tilde{\kappab}^\top \tilde{\w} = -\kappab^\top \w \leq \psip^*(\kappab)
\psip(\w)$,where the first inequality is easy to show when observing
that~$\tilde{\GG}_p \subseteq \GG_p$, and the last inequality is the
generalized H\"older inequality for a norm and its dual-norm.
Since~$\psip^*(\kappab) \psip(\w) \leq \lambda \psip(\w)$ we have in fact
equality, and we conclude the proof.
\end{proof}